\documentclass{article} %
\usepackage{arxiv,times}

\usepackage[colorlinks]{hyperref}
\usepackage{url}
\usepackage{amsfonts}
\usepackage{mathtools}
\usepackage{amsthm}
\usepackage{cleveref}
\usepackage{subcaption}
\usepackage{comment}
\usepackage{wrapfig}
\usepackage{booktabs}
\usepackage{mwe}
\usepackage{enumitem}

\usepackage{amsmath,amsfonts,bm}

\def\eqref#1{equation~\ref{#1}}

\def\1{\bm{1}}

\def\vb{{\bm{b}}}

\def\mA{{\bm{A}}}

\def\mC{{\bm{C}}}

\DeclareMathAlphabet{\mathsfit}{\encodingdefault}{\sfdefault}{m}{sl}
\SetMathAlphabet{\mathsfit}{bold}{\encodingdefault}{\sfdefault}{bx}{n}

\newcommand{\E}{\mathbb{E}}

\newcommand{\R}{\mathbb{R}}

\newcommand{\KL}{D_{\mathrm{KL}}}

\newcommand{\Cov}{\mathrm{Cov}}

\newcommand{\BHt}{B_H(t)}

\newcommand{\BHonet}{B_H^{(I)}(t)}
\newcommand{\BHones}{B_H^{(I)}(s)}

\newcommand{\BHtwot}{B_H^{(II)}(t)}
\newcommand{\BHtwos}{B_H^{(II)}(s)}
\newcommand{\BHonetwo}{B_H^{(I, II)}}
\newcommand{\BHonetwot}{B_H^{(I, II)}(t)}

\newcommand{\Kone}{K^{(I)}}
\newcommand{\Ktwo}{K^{(II)}}

\newcommand{\Konets}{K^{(I)}(t,s)}
\newcommand{\Ktwots}{K^{(II)}(t,s)}

\newcommand*\diff{\mathop{}\!\mathrm{d}}

\newcommand{\dt}{\diff{t}}
\newcommand{\du}{\diff{u}}
\newcommand{\dgm}{\diff{\gamma}}
\newcommand{\dWt}{\diff{W(t)}}
\newcommand{\dWs}{\diff{W(s)}}
\newcommand{\dXt}{\diff{X(t)}}
\newcommand{\dZt}{\diff{Z(t)}}
\newcommand{\dYkt}{\diff{Y_k(t)}}

\newcommand{\dBHt}{\diff{B_H(t)}}
\newcommand{\dMABHt}{\diff{\hat{B}_H(t)}}

\newcommand{\postZ}{\tilde{Z}}
\newcommand{\postX}{\tilde{X}}

\crefname{equation}{eq.}{eq.}
\Crefname{equation}{Eq.}{Eq.}
\crefname{theorem}{thm.}{thms.}
\Crefname{Theorem}{Thm.}{Thms.}
\crefname{proposition}{prop.}{props.}
\Crefname{proposition}{Prop.}{Props.}
\crefname{definition}{dfn.}{dfn.}
\Crefname{definition}{Dfn.}{Dfn.}
\crefname{remark}{remark}{remark}
\Crefname{Remark}{Remark}{Remark}
\Crefname{algorithm}{Alg.}{Alg.}

\newtheorem{thm}{Theorem}
\newtheorem{remark}{Remark}

\newtheorem{prop}{Proposition}
\newtheorem{dfn}{Definition}

\newcommand{\eg}{\textit{e}.\textit{g}. }

\newcommand{\ie}{\textit{i}.\textit{e}. }
\newcommand{\cf}{\textit{cf}.} 

\crefname{section}{Sec.}{Secs.}
\Crefname{section}{Sec.}{Secs.}
\crefname{equation}{Eq.}{Eqs.}
\Crefname{equation}{Eq.}{Eqs.}
\crefname{figure}{Fig.}{Figs.}
\Crefname{figure}{Fig.}{Figs.}
\crefname{table}{Tab.}{Tabs.}
\Crefname{table}{Tab.}{Tabs.}
\crefname{thm}{Thm.}{Thms.}
\Crefname{thm}{Thm.}{Thms.}
\crefname{dfn}{Dfn.}{Dfns.}
\crefname{dfn}{Dfn.}{Dfns.}
\crefname{remark}{remark}{remarks}
\Crefname{Remark}{Remark}{Remarks}
\crefname{prop}{Prop.}{Prop.}
\Crefname{prop}{Prop.}{Prop.}
\Crefname{algorithm}{Alg.}{Alg.}
\crefname{appendix}{App.}{apps.}
\Crefname{appendix}{App.}{Apps.}
\crefname{appsec}{appendix}{appendices}
\Crefname{appsec}{Appendix}{Appendices}

\renewcommand{\paragraph}[1]{{\vspace{1mm}\noindent \bf #1}.}

\newcommand{\expE}[1]{\E\left[ #1 \right]}

\definecolor{dark-blue}{rgb}{0.15,0.15,0.4}
\definecolor{medium-blue}{rgb}{0,0,0.5}
\hypersetup{
  colorlinks, linkcolor={dark-blue},
  citecolor={dark-blue}, urlcolor={medium-blue}
}

\newcommand{\oneh}{\frac{1}{2}}

\title{Variational Inference for SDEs Driven by Fractional Noise}

\author{Rembert Daems$^{1,2}$
  \And
  Manfred Opper$^{3,4,5}$
  \And
  Guillaume Crevecoeur$^{1,2}$
  \And
  Tolga Birdal$^{6}$
  \and
\hspace{0.5mm}\small{$^1$ Ghent University},\,
\hspace{-1.1mm}\qquad\qquad\small{$^2$ Flanders Make},\,
\quad\qquad\small{$^3$ Technical University of Berlin},\,\\
\hspace{-7mm}\noindent\small{$^4$ University of Birmingham},\,
\small{$^5$ University of Potsdam},\,
\small{$^6$ Imperial College  London} 
}

\iclrfinalcopy %
\begin{document}

\maketitle

\begin{abstract}
We present a novel variational framework for performing inference in (neural) stochastic differential equations (SDEs) driven by Markov-approximate fractional Brownian motion (fBM). SDEs offer a versatile tool for modeling real-world continuous-time dynamic systems with inherent noise and randomness. Combining SDEs with the powerful inference capabilities of variational methods, enables the learning of representative function distributions through stochastic gradient descent. However, conventional SDEs typically assume  the underlying noise to follow a Brownian motion (BM), which hinders their ability to capture long-term dependencies. In contrast, fractional Brownian motion (fBM) extends BM to encompass non-Markovian dynamics, but existing methods for inferring fBM parameters are either computationally demanding or statistically inefficient. 
In this paper, building upon the Markov approximation of fBM, we derive the evidence lower bound essential for efficient variational inference of posterior path measures, drawing from the well-established field of stochastic analysis. Additionally, we provide a closed-form expression to determine optimal approximation coefficients. Furthermore, we propose the use of neural networks to learn the drift, diffusion and control terms within our variational posterior, leading to the variational training of neural-SDEs. In this framework, we also optimize the Hurst index, governing the nature of our fractional noise. Beyond validation on synthetic data, we contribute a novel architecture for variational latent video prediction,—an approach that, to the best of our knowledge, enables the first variational neural-SDE application to video perception.

\end{abstract}

\vspace{-2mm}\section{Introduction}\vspace{-2mm}
Our surroundings constantly evolve over time, influenced by several dynamic factors, manifesting in various forms, from the weather patterns and the ebb \& flow of financial markets to the movements of objects \& observers, and the subtle deformations that reshape our environments~\citep{gojcic2021weakly,rempe2021humor}.
Stochastic differential equations (SDEs) provide a natural way to capture the randomness and continuous-time dynamics inherent in these real-world processes.
To extract meaningful information about the underlying system, \ie to infer the model parameters and to accurately predict the unobserved paths, variational inference (VI)~\citep{bishop2006pattern} is used as an efficient means, computing the posterior probability measure over paths~\citep{opper2019variational,li2020scalable,ryder2018black}\footnote{KL divergence between two SDEs over a finite time horizon has been well-explored in the control literature \citep{theodorou2015nonlinear,kappen2016adaptive}.}. 

The traditional application of SDEs assumes that the underlying noise processes are generated by standard Brownian motion (BM) with independent increments.
Unfortunately, for many practical scenarios, BM falls short of capturing the full complexity and richness of the observed real data, which often contains long-range dependencies, rare events, and intricate temporal structures that cannot be faithfully represented by a Markovian process. 
The non-Markovian fractional Brownian motion (fBM)~\citep{mandelbrot1968fractional} extends BM to stationary increments with a more complex dependence structure, \ie long-range dependence vs. roughness/regularity controlled by its \emph{Hurst index}~\citep{gatheral2018volatility}.
Yet, despite its desirable properties, the computational challenges and intractability of analytically working with fBMs pose significant challenges for inference.

In this paper, we begin by providing a tractable variational inference framework for SDEs driven by fractional Brownian motion (Types I \& II). To this end, we benefit from the relatively under-explored \emph{Markov representation} of fBM and path-wise approximate fBM through a linear combination of a finite number of Ornstein--Uhlenbeck (OU) processes driven by a common noise~\citep{carmona1998fractional,carmona1998simultaneous,harms2019affine}. We further introduce a differentiable method to optimise for the associated coefficients and conjecture (as well as empirically validate) that this \emph{strong} approximation enjoys super-polynomial convergence rates, allowing us to use a handful of processes even in complex problems.

Such \emph{Markov-isation} also allows us to inherit the well-established tools of traditional SDEs including Girsanov's change of measure theorem~\citep{oksendal2003stochastic}, which we use to derive and maximise the corresponding \emph{evidence lower bound} (ELBO) to yield posterior path measures as well as maximum likelihood estimates as illustrated in~\cref{fig:teaser}.  
We then use our framework in conjunction with neural networks to devise VI for neural-SDEs~\citep{liu2019neural,li2020scalable} driven by the said fractional diffusion. We deploy this model along with a novel neural architecture for the task of enhanced video prediction. To the best of our knowledge, this is the first time either fractional or variational neural-SDEs are used to model videos.
Our contributions are:
\vspace{-2.5mm}

\begin{itemize}[noitemsep,leftmargin=*,topsep=0em]
    \item We make accessible the relatively uncharted Markovian embedding of the fBM and its strong approximation, to the machine learning community. This allows us to employ the traditional machinery of SDEs in working with non-Markovian systems.
    \item We show how to balance the contribution of Markov processes by optimising for the combination coefficients in closed form. We further estimate the (time-dependent) {Hurst index} from data. 
    \item We derive the {evidence lower bound} for SDEs driven by approximate fBM of both Types I and II.
    \item We model the drift, diffusion and control terms in our framework by neural networks, and propose a novel architecture for video prediction.
\end{itemize}
\vspace{-1.5mm}
We will make our implementation publicly available upon publication.

\begin{figure}[t]
    \centering
    \begin{subfigure}{0.45\textwidth}
        \centering
        \includegraphics[width=\textwidth]{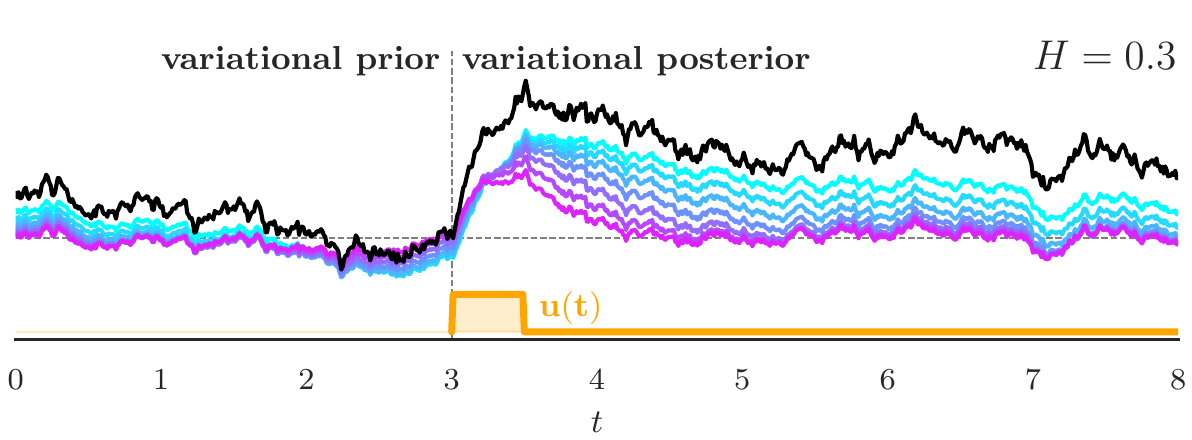}
    \end{subfigure}\hfill
    \begin{subfigure}{0.54\textwidth}
        \centering
        \includegraphics[width=\textwidth]{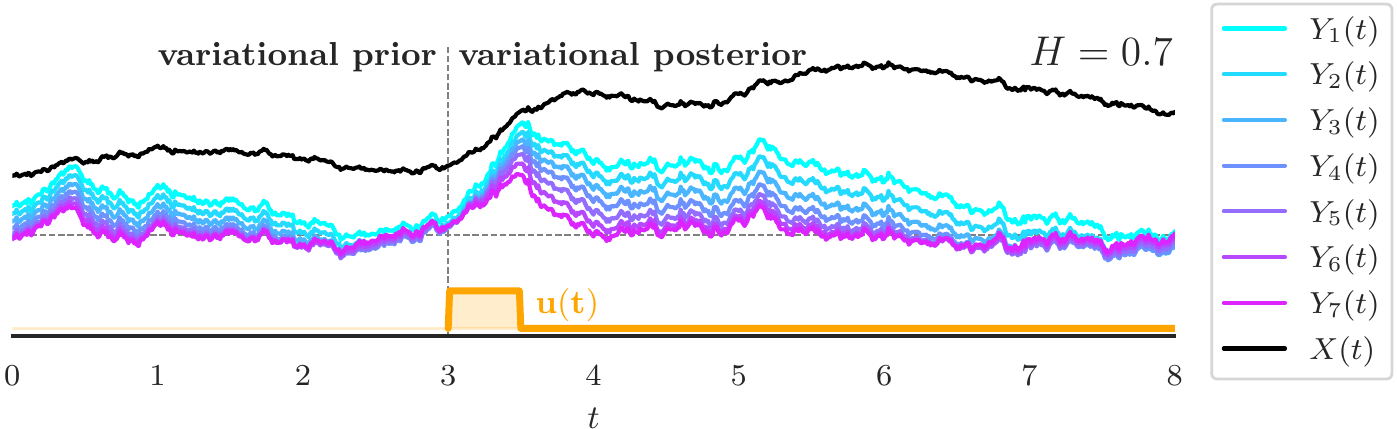}
    \end{subfigure}\hfill
    \caption{We leverage the Markov approximation, where the non-Markovian fractional Brownian motion with Hurst index $H$ is approximated by a linear combination of a finite number of Markov processes ($Y_1(t),\dots,Y_K(t)$), and propose a variational inference framework in which the posterior is steered by a control term $u(t)$. 
    Note the long-term memory behaviour of the processes, where individual $Y_k(t)$s have varying
    transient effects, from $Y_1(t)$ having the longest memory to $Y_7(t)$ the shortest, and tend to forget the action of $u(t)$ after a certain time frame.\vspace{-4mm}
    }
    \label{fig:teaser}
\end{figure}

\vspace{-2.5mm}\section{Related Work}
\label{sec:related}
\vspace{-3mm}
\paragraph{Fractional noises and neural-SDEs}
fBM~\citep{mandelbrot1968fractional} was originally used for the simulation of rough volatility in finance~\citep{gatheral2018volatility}. 
Using the Lemarié-Meyer wavelet representation,~\cite{allouche2022generative} provided a large probability bound on the deep-feedforward RELU network approximation of fBM, where up to log terms, a uniform error of $O(N^{-H})$ is achievable with $\log(N)$ hidden layers and $O(N)$ parameters. 
\cite{tong2022learning} approximated the fBM (only Type II) with sparse Gaussian processes. Unfortunately, they are limited to Euler-integration and to the case of $H>1/3$.
Their model was also not applied to videos.
Recently, \cite{yang2023neural} applied Levy driven neural-SDEs to times series prediction and \cite{hayashi2022fractional} considered neural-SDEs driven by fractional noise. Neither of those introduce a variational framework. 
Both \cite{liao2019learning,morrill2021neural} worked with \emph{rough path theory} to model long time series via rough neural-SDEs.
To the best of our knowledge, we are the firsts to devise a VI framework for neural-SDEs driven by a path-wise (strong) approximation of fBM.
\vspace{-1mm}

\paragraph{SDEs and visual understanding}
Apart from the recent video diffusion models~\citep{luo2023videofusion,yang2022diffusion,ho2022video}, SDEs for spatiotemporal visual generation is relatively unexplored. \cite{park2021vid,ali2023vidstyleode} used neural-ODEs to generate and manipulate videos while~\citep{rempe2020caspr} used neural-ODEs for temporal 3D point cloud modeling. SDENet~\citep{kong2020sde} and MDSDE-Net~\citep{zhang2023milstein} learned drift and diffusion networks for uncertainty estimation of images using out-of-distribution data. \cite{tong2022learning} used approximate-fBMs in score-based diffusion modeling for image generation. \cite{gordon2021latent} briefly evaluated different neural temporal models for video generation. 
While~\cite{babaeizadeh2018stochastic} used VI for video prediction, they did not employ SDEs. To the best of our knowledge, we are the firsts to use neural-SDEs in a variational framework for video understanding.

\section{Background}
\label{sec:bg}
We first tailor and make accessible the fractional Brownian Motion (fBM) and its relatively less explored Markov approximations for the learning community. We then describe the SDEs driven by fBM and its approximation before delving into the inference. We leave the proofs to our appendix.
 
\subsection{Fractional Brownian Motion (fBM) \& Its Markov Approximation}
\label{sec:fbm}
\begin{dfn}[Fractional Brownian Motion (Types I \& II)]
    fBM is a self-similar, non-Markovian, non-martingale, zero-mean Gaussian process $\left(\BHt\right)_{t\in[0,T]}$ for $T>0$ with a covariance of either 
    \begin{align}
        \label{eq:cov-type1}
        \E\left[\BHonet\BHones\right] &= \oneh (|t|^{2H} + |s|^{2H} - |t-s|^{2H})\qquad\qquad\qquad\qquad&\mathrm{(Type\,\,I)}\\
        \E\left[\BHtwot\BHtwos\right] &= \frac{1}{\Gamma^2(H+1/2)} \int_0^s ((t-u)(s-u))^{H-1/2} \du \qquad &\mathrm{(Type\,\,II)}
        \label{eq:cov-type2}
    \end{align}
    where $t > s$, $0<H<1$ is the \emph{Hurst index}, superscripts denote the types and $\Gamma$ is the Gamma function.
\end{dfn}
fBM recovers \emph{Brownian motion} (BM) for $H=1/2$ (regular diffusion) and generalizes it for other choices. The increments are (i) positively correlated for $H>1/2$ (super-diffusion) where the tail behaviour is infinitely heavier than that of BM, and (ii) negatively correlated for $H<1/2$ (sub-diffusion), with variance $\E\left( |\BHonet-\BHones|^2 \right)=|t-s|^{2H}$ for Type I. The Type II model implies nonstationary increments of which the marginal distributions are dependent on
the time relative to the start of the observed sample, \ie all realizations  would have to be found very close to the unconditional mean (\ie, the origin)~\citep{lim1995asymptotic,davidson2009type}.

\begin{dfn}[Integral representations of fBM]
$\BHonetwo$ admit the following integral forms due to the Mandelbrot van-Ness and Weyl representations, respectively~\citep{mandelbrot1968fractional}:
\begin{align}
\BHonet &= \frac{1}{\Gamma(H+1/2)} \int_{-\infty}^t \left[\Konets\vcentcolon= \left((t-s)^{H-1/2} - (-s)_{+}^{H-1/2}\right)\right] \dWs\\
 &= \frac{1}{\Gamma(H+1/2)} \left( \int_{-\infty}^0 \left((t-s)^{H-1/2} - (-s)^{H-1/2}\right) \dWs + \int_{0}^t (t-s)^{H-1/2} \dWs \right) \nonumber\\
\BHtwot &= \frac{1}{\Gamma(H+1/2)} \int_0^t \left[\Ktwots\vcentcolon= (t-s)^{H-1/2}\right] \dWs
\end{align}
\end{dfn}
where $\Kone$ and $\Ktwo$ are the kernels corresponding to Types I and II, respectively. 

\begin{prop}[Markov representation of fBM~\citep{harms2019affine}]\label{dfn:MarkovFBM}
    The long memory processes $\BHonetwot$ can be represented by an infinite linear combination of Markov processes, all driven by the same Wiener noise, but with different time scales, defined by speed of mean reversion $\gamma$. For both types we have representations of the form: 
    \begin{equation}\label{Markov_rep}
    B_H(t) =
    \begin{dcases}
        \int_0^\infty (Y_\gamma(t) - Y_\gamma(0) \mu(\gamma) \dgm, \quad H < 1/2, \\
        - \int_0^\infty \partial_\gamma (Y_\gamma(t) - Y_\gamma(0) \nu(\gamma) \dgm, \quad H > 1/2
    \end{dcases},
\end{equation}
where $\mu(\gamma) = \gamma^{-(H+1/2)}/\left({\Gamma(H+1/2) \Gamma(1/2-H)}\right)$ and $\nu(\gamma) = \gamma^{-(H-1/2)}/(\Gamma(H+1/2)$ $\Gamma(3/2-H))$. Note, these non--negative densities are not normalisable. To simplify notation, we will drop explicit dependency on the types $(I, II)$ in what follows.
For each $\gamma\geq 0$, and for both types $I$ and $II$, the processes $Y_\gamma(t)$ are OU processes which are solutions to the SDE $d Y_\gamma(t) = - \gamma Y_\gamma(t) \dt + \dWt$. This SDE is solved by
\begin{equation}
Y_\gamma(t) = Y_\gamma(0) e^{-\gamma t} + \int_0^t e^{-\gamma (t-s)} \dWs.
\label{eq:solved-y-sde}
\end{equation}
"Type I" and  "Type II" differ in the initial conditions $Y_\gamma(0)$. One can show that:
\begin{equation}
    Y^{(I)}_\gamma(0) = \int_{-\infty}^0 e^{\gamma s} \dWs \qquad \mathrm{and} \qquad Y^{(II)}_\gamma(0) = 0.
    \label{eq:y-0}    
\end{equation}

\end{prop}
\begin{dfn}[Markov approximation of fBM (MA-fBM)]\label{remark:markov}
\Cref{Markov_rep} suggests that $B_H(t)$ could be well approximated by a Markov process $\hat{B}_H(t)$ by $(i)$ truncating the integrals at finite $\gamma$ values $(\gamma_1...\gamma_K)$ and $(ii)$ approximating the integral by a numerical quadrature as a finite linear combination involving quadrature points and weights $\{\omega_k\}$. Changing the notation $Y_{\gamma_k}(t) \to Y_k(t)$:
\begin{equation}
\label{eq:v}
B_H(t) \approx \hat{B}_H(t) \equiv \sum_{k=1}^K \omega_k \left(Y_k(t) - Y_k(0)\right),
\end{equation}
where for fixed $\gamma_k$ the choice of $\omega_k$ depends on $H$ and the choice of "Type I" or "Type II".  For "Type II", we set
$Y_k(0) = 0$.
Since $Y_k(t)$ is normally distributed~\cite[Thm. 2.16]{harms2019affine} and can be assumed stationary for "Type I",  
we can simply sample $\left(Y_1^{(I}(0), \dots, Y_K^{(I)}(0)\right)$ from a normal distribution with mean $\bm{0}$ and covariance $\mC_{i,j} = {1}/{(\gamma_i + \gamma_j)}$ (see~\cref{eq:cov-y0}).
\end{dfn}
This strong approximation provably bounds the sample paths:
\begin{thm}[\cite{alfonsi2021approximation}]\label{thm:alfonsi}
    For rough kernels ($H<1/2$) and $\{\omega_k\}$ following a Gaussian quadrature rule, there exists a constant $c$ per every $t\in(0,T)$ such that:
    \begin{equation}
        \E|B_H^{(II)}(t) - \hat{B}_H^{(II)}(t)| \leq O(K^{-cH}), \quad \mathrm{where} \quad 1<c\leq 2,
    \end{equation}
    as $K\to\infty$. Note that, in our setting, $B_H^{(II)}(0)=\hat{B}_H^{(II)}(0)=0$. %
\end{thm}

In the literature, different choices of $\gamma_k$ and $\omega_k$ have been proposed \citep{harms2019affine,carmona1998fractional,carmona2000approximation} and for certain choices, it is possible to obtain a superpolynomial rate, as shown by~\cite{bayer2023markovian} for the Type II case.
As we will show in~\cref{sec:omega}, choosing $\gamma_k = r^{k-n}, k=1,\ldots,K$ with $n=(K+1)/2$~\citep{carmona1998fractional}, we will optimise $\{\omega_k\}_k$ for both types, to get optimal rates.

\subsection{SDEs driven by (fractional) BM}
\begin{dfn}[SDEs driven by BM (BMSDE)]\label{dfn:BMSDE}
    A common generative model for stochastic dynamical systems considers a set of observational data ${\cal{D}} = \{O_1,\ldots, O_N\}$, 
where the $O_i$ are generated 
(conditionally) independent at random at discrete times $t_i$ with a likelihood $p_\theta\left(O_{i} \mid X(t_i)\right)$. The prior 
information about the unobserved path $\{X(t); t\in [0, \; T] \} $ of the latent process $X(t) \in \R^M$ is given by the assumption that
$X(t)$ fulfils the SDE: 
\begin{equation}\label{prior_process_Wiener}
    \mathrm{~d} X(t) = b_\theta \left(X(t), t\right) \dt +\sigma_\theta\left(X(t), t\right) \mathrm{d} W(t)
\end{equation}
The \textit{drift function} $b_\theta \left(X, t\right)\in \R^D$  models the deterministic part of the change $\mathrm{~d} X(t)$ of the state variable 
$X(t)$ during the infinitesimal time interval $\dt$, whereas the \textit{diffusion matrix} $\sigma_\theta\left(X(t), t\right) \in \R^{D\times D}$ 
(assumed to be symmetric and non--singular, for simplicity) encodes
the strength of the added Gaussian \textit{white noise} process, where $\mathrm{d} W(t) \in \R^D$ is the infinitesimal increment of a vector of independent Wiener processes during $\dt$. 
\end{dfn}

\begin{dfn}[SDEs driven by fBM (fBMSDE)]\label{dfn:fBMSDE}
    \cref{dfn:BMSDE} can be formally extended to the case of fractional Brownian motion replacing $\dWt$ by $\dBHt$~\citep{guerra2008stochastic}:
\begin{equation}\label{eq:prior_process_frac}
    \mathrm{~d} X(t) = b_\theta \left(X(t), t\right) \dt +\sigma_\theta\left(X(t), t\right) \mathrm{d} \BHt.
\end{equation}
\end{dfn}

\begin{remark}
    Care must be taken in a proper definition of the diffusion part in the fBMSDE \cref{eq:prior_process_frac} and in developing
appropriate numerical integrators for simulations, when the diffusion 
$\sigma_\theta\left(X(t), t\right)$ explicitly depends on the state $X(t)$. Corresponding stochastic integrals of the Itô type cannot be applied when $H < 1/2$ and other approaches (which are generalisations of the Stratonovich SDE for $H=\frac{1}{2}$) are necessary~\citep{lysy2013statistical}.
\end{remark}

\section{Method}\label{sec:method}
Our goal is to extend variational inference (VI)~\cite{bishop2006pattern} to the case where the Wiener process in~\cref{prior_process_Wiener} is replaced by an fBM as in~\cref{dfn:fBMSDE}. 
Unfortunately, the processes defined by \cref{eq:prior_process_frac} are not Markovian preventing us from resorting to the standard Girsanov change of measure approach known for "ordinary" SDE to compute KL--divergences and ELBO functionals needed for VI~\citep{opper2019variational}. While~\cite{tong2022learning} leverage sparse approximations for Gaussian processes, this makes $B_H$ \textit{conditioned} on a finite but larger number of so--called \emph{inducing variables}. %
We take a completely different and conceptually simple approach to VI for fBMSDE based on the exact representation of $B_H(t)$ given in~\cref{dfn:MarkovFBM}.
To this end, we first show how the \emph{strong} Markov-approximation in~\cref{remark:markov} can be used to approximate an SDE driven by fBM, before delving into the VI for the \emph{Markov-Approximate fBMSDE}.
\begin{dfn}[Markov-Approximate fBMSDE (MA-fBMSDE)]\label{dfn:MAfBMSDE}
    Substituting the fBM, $\BHt$, in~\cref{dfn:fBMSDE} by the finite linear combination of OU-processes $\hat{B}_H(t)$, we define MA-fBMSDE as:
    \begin{equation}\label{eq:MAfBMSDE}
     \dXt  = b_\theta \left(X(t), t\right) dt +\sigma_\theta\left(X(t), t\right) \dMABHt,
\end{equation}
where $\dMABHt = \sum_{k=1}^K \omega_k \dYkt$ with $\dYkt = - \gamma Y_k (t) dt + \dWt$ (\cf~\cref{remark:markov}).
\end{dfn}

\begin{prop}
    $X(t)$ can be augmented by the finite number of Markov processes $Y_k(t)$ (approximating $\BHt$) to a higher dimensional state variable of the form $Z(t) \doteq \left(X(t), Y_1(t), \ldots Y_K(t)\right) \in R^{D(K+1)}$, such that the joint process of the augmented system becomes Markovian and can be described by an 'ordinary' SDE:
    \begin{equation}\label{eq:prior_process}
    \dZt = h_\theta \left(Z(t), t\right) \dt +\Sigma_\theta\left(Z(t), t\right) \dWt,
\end{equation}
where the augmented drift vector $h_\theta \in R^{D \times(K+1)}$ and the diffusion matrix $\Sigma_\theta \left(Z, t\right)\in \R^{D (K+1) \times D}$ are given by
\begin{equation}\label{eq:hSigma}
h_\theta \left(Z, t\right)  =
\left(\begin{array}{c}
b_\theta \left(X, t\right) - \sigma_\theta\left(X, t\right)\sum_k \omega_k \gamma_k Y_k \\
- \gamma_1 Y_1\\
\ldots \\
-  \gamma_K Y_K\\
\end{array}
\right)\qquad
\Sigma_\theta \left(Z, t\right)  = \left(\begin{array}{cc@{\hspace{0.3em}}c@{\hspace{0.3em}}c}
\bar{\omega}\sigma_\theta (X ,t) \\
\vec{1} \\
     \vdots  \\
\vec{1}
\end{array}
\right),
\end{equation}
where $\vec{1} = (1,1, \ldots,1)^\top \in R^D$. We will refer to~\cref{eq:prior_process} as the \emph{variational prior}.
\end{prop}
\begin{proof}
    Each of the $D$ components of the vectors $Y_k$ use the same scalar weights $\omega_k\in \R$. Also, note that each $Y_k$ is driven by the \textit{same} vector of 
Wiener processes. Hence, we obtain the system of SDEs given by
\begin{equation}
\begin{split}
     \dXt & = b_\theta \left(X(t), t\right) \dt - \sigma_\theta\left(X(t), t\right) \sum_k \omega_k \gamma_k Y_k(t) \dt 
     + \bar{\omega}  \sigma_\theta\left(X(t), t\right) \dWt \\
      \dYkt  &  = - \gamma_k Y_k(t) \dt + \dWt \qquad \mbox{for}\qquad k=1, \ldots, K
\end{split}
\end{equation}
where $\bar{\omega} \doteq \sum_k \omega_k$. 
This system of equations can be collectively represented in terms of the augmented variable $Z(t) \vcentcolon=  \left(X(t), Y_1(t), \ldots Y_K(t)\right) \in R^{D(K+1)}$ leading to a single SDE specified by~\cref{eq:prior_process,eq:hSigma}.
\end{proof}

\Cref{eq:prior_process} represents a standard SDE driven by Wiener noise allowing us to utilise the standard tools of stochastic analysis, such as the Girsanov change of measure theorem and derive the \emph{evidence lower bounds} (ELBO) required for VI. This is what we will exactly do in the sequel.

\begin{prop}[Controlled MA-fBMSDE]\label{prop:U}
The paths of~\cref{eq:prior_process} can be steered by adding a control term $u(X, Y_1,\ldots,Y_K, t) \in R^D$ that depends on all variables to be optimised, to the drift $h_\theta$ resulting in the transformed SDE, a.k.a. the \emph{variational posterior}:
\begin{equation}\label{eq:posterior_SDE}
    \mathrm{~d} \postZ(t) = \left(h_\theta\left(\postZ(t), t\right)+  \sigma_{\theta}(\postZ(t), t) u(\postZ(t), t) \right) \dt + \Sigma_\theta\left(\postZ(t), t\right) \mathrm{d} W(t)
\end{equation}
\end{prop}
\begin{proof}[Sketch of the proof]
Using the fact that the posterior probability measure over paths $\postZ(t)$ $\{\postZ(t); t\in [0, \; T]\} $ is absolutely continuous w.r.t. the prior process, we apply the Girsanov theorem (\cf~\cref{app:girsanov2}) on~\cref{eq:prior_process} to write the new drift, from which the posterior SDE in~\cref{eq:posterior_SDE} is obtained.
\end{proof}

We will refer to~\cref{eq:posterior_SDE} as the \emph{variational posterior}. In what follows, we will assume a parametric form for the control function $u(\postZ(t), t) \equiv u_\phi(\postZ(t), t)$ (as \eg given by a neural network) and will devise a scheme for inferring the \emph{variational parameters $(\theta,\phi)$}, \ie variational inference.

\begin{prop}[Variational Inference for MA-fBMSDE]
    The \textit{variational parameters} $\phi$ are optimised by minimising the KL--divergence between the posterior and the prior, where the corresponding \textit{evidence lower bound} (ELBO) to be maximised is:
\begin{equation}
\label{eq:brownian-elbo}
\log p\left(O_1, O_2, \ldots, O_N \mid \theta\right) \geq 
\mathbb{E}_{\postZ_u}\left[\sum_{i=1}^N \log p_\theta\left(O_i \mid \postZ(t_i) \right)-\int_0^T \frac{1}{2}\left\|u_\phi\left(\postZ(t), t\right)\right\|^2 \mathrm{~d} t\right],
\end{equation}
where the observations $\{O_i\}$ are included by likelihoods $p_\theta\left(O_{i} \mid \postZ(t_i)\right)$ and the expectation is taken over random paths of the approximate posterior process defined by 
(\cref{eq:posterior_SDE}).
\end{prop}\vspace{-2mm}
\begin{proof}[Sketch of the proof]
    Since we can use Girsanov's theorem II~\citep{oksendal2003stochastic}, the variational bound derived in~\cite{li2020scalable} (App. 9.6.1) directly applies.
\end{proof}

\begin{remark}
It is noteworthy that the measurements with their likelihoods $p_\theta\left(O_{i} \mid \postX(t_i)\right)$ depend only on the component $\postX(t)$ of the augmented state $\postZ(t)$. The additional variables $Y_k(t)$ which are used to model the noise in the SDE are not directly observed. However, computation of the ELBO requires initial values for all state variables $\postZ(0)$ (or their distribution). Hence, we sample $Y_k(0)$ in accordance with~\cref{remark:markov}. 
\end{remark}
\subsection{Optimising the approximation}\label{sec:omega}\vspace{-1mm}
We now present the details of our novel method for optimising our approximation
$ \hat{B}^{(I,II)}_H(t)$ for $\omega_k$. 
To this end, we first follow~\cite{carmona1998fractional} and choose a geometric sequence of $\gamma_k = (r^{1-n},r^{2-n},\dots,r^{K-n}), n=\frac{K+1}{2}, r>1$. 
Rather than relying on methods of numerical quadrature, we consider a simple
measure for the quality of the approximation over a fixed time interval $[0, \; T]$ which can be \textit{optimised analytically} for both types I and II.
\begin{prop}[Optimal ${\bm{\omega}\doteq[\omega_1, \ldots, \omega_K]}$ for $\hat{B}^{(I,II)}(t)$]
The $L_2$-error of our approximation
\begin{equation}
    {\cal{E}}^{(I,II)}(\bm{\omega}) = \int_0^T \E\left[\left(\hat{B}^{(I,II)}_H(t) - \BHonetwot\right)^2\right]\dt
\end{equation}
is minimized at ${\mA^{(I,II)}} \bm{\omega} = \vb^{(I,II)}$, where
    \begin{eqnarray}
    \mA^{(I)}_{i,j} &=& \frac{2T + \frac{e^{-\gamma_i T}-1}{\gamma_i} + \frac{e^{-\gamma_j T}-1}{\gamma_j}}{\gamma_i + \gamma_j}\,, \,\,\,\, \mA^{(II)}_{i,j} = \frac{T + \frac{e^{-(\gamma_i+\gamma_j)T}-1}{\gamma_i + \gamma_j}}{\gamma_i + \gamma_j}\\
    \vb^{(I)}_k &=&
        \frac{2T}{\gamma_k^{H+1/2}}
        - \frac{T^{H+1/2}}{\gamma_k \Gamma(H+3/2)}
        + \frac{e^{-\gamma_k T} - Q(H+1/2, \gamma_k T) e^{\gamma_k T}}{\gamma_k^{H+3/2}}
    \label{eq:omega-bI}
    \\
    \vb^{(II)}_k &=&
        \frac{T}{\gamma_k^{H+1/2}} P(H + 1/2, \gamma_k T) - \frac{H+1/2}{\gamma_k^{H+3/2}} P(H+3/2, \gamma_k T).
\end{eqnarray}
$P(z, x)=\frac{1}{\Gamma(z)}\int_0^x t^{z-1}e^{-t}\dt$ is the regularized lower incomplete gamma function and $Q(z,x)=\frac{1}{\Gamma(z)}\int_x^\infty t^{z-1}e^{-t}\dt$ is the regularized upper incomplete gamma function.
\end{prop}
\begin{proof}[Sketch of the proof]
By expanding the $L_2$-error we find a tractable quadratic form of the criterion:
\begin{align}
    {\cal{E}}^{(I,II)}(\bm{\omega}) &= \int_0^T \E\left[\left(\hat{B}^{(I,II)}_H(t) - \BHonetwot\right)^2\right] \dt \\
    &=  \int_0^T \left(\expE{\hat{B}^{(I,II)}_H(t)^2} + \expE{\BHonetwot^2} - 2 \expE{\hat{B}^{(I,II)}_H(t) \BHonetwot} \right) \dt \nonumber\\
    &= \bm{\omega}^T \mA^{(I,II)} \bm{\omega} - 2 {\vb^{(I,II)}}^T \bm{\omega} + \mathrm{const},\nonumber
\end{align}
whose non-trivial minimum is attained as the solution to the system of equations ${\mA^{(I,II)}}\bm{\omega} = \vb^{(I,II)}$.
We refer the reader to \cref{app:optimized-omega} for the full proof and derivation. 
\end{proof}

\begin{figure}[t]
    \centering
    \begin{subfigure}{0.24\textwidth}
        \centering
        \includegraphics[width=\textwidth,trim=6 8 7 7,clip]{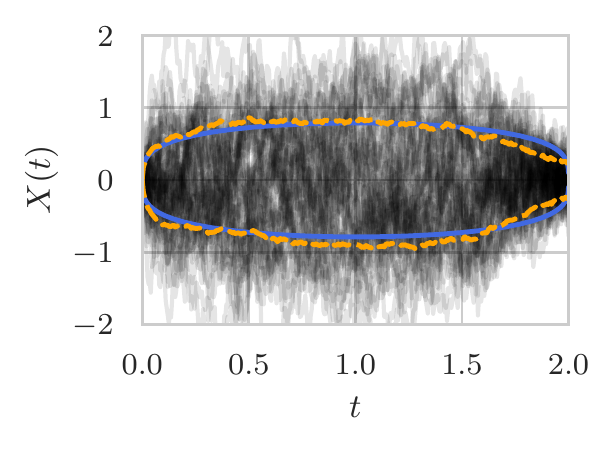}
        \caption{$H = 0.3, \theta=0.0$}
    \end{subfigure}\hfill
    \begin{subfigure}{0.24\textwidth}
        \centering
        \includegraphics[width=\textwidth,trim=6 8 7 7,clip]{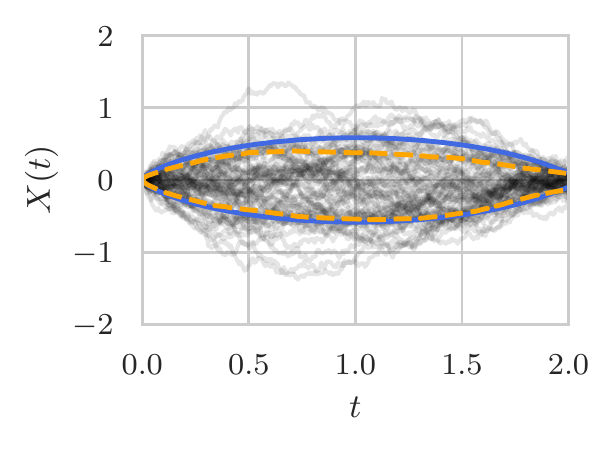}
        \caption{$H = 0.7, \theta=0.0$}
    \end{subfigure}\hfill
    \begin{subfigure}{0.24\textwidth}
        \centering
        \includegraphics[width=\textwidth,trim=6 8 7 7,clip]{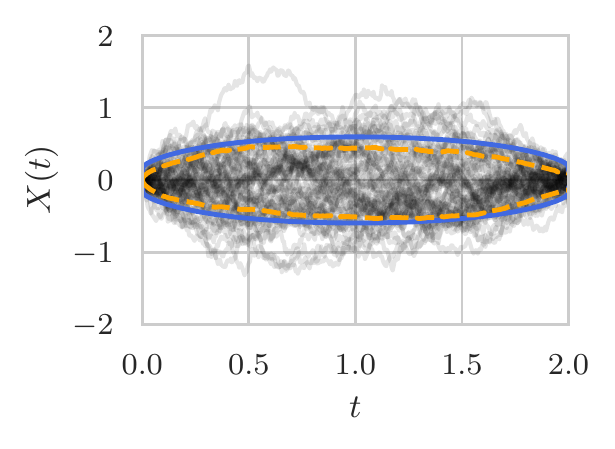}
        \caption{$H = 0.6, \theta=1.0$}
    \end{subfigure}\hfill
    \begin{subfigure}{0.24\textwidth}
        \centering
        \includegraphics[width=\textwidth,trim=6 8 7 7,clip]{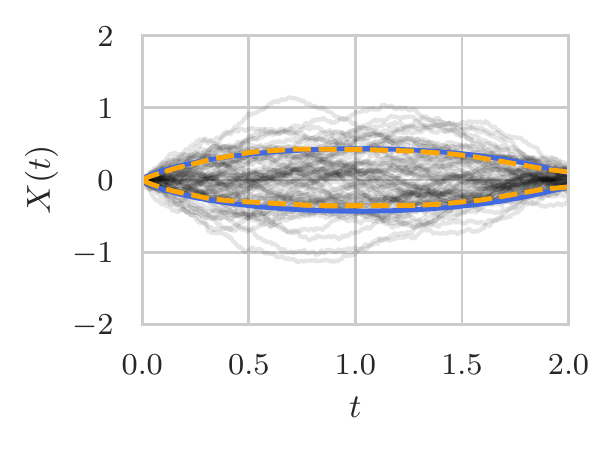}
        \caption{$H = 0.8, \theta=1.0$}
    \end{subfigure}
    \caption{The true variance (blue) of a fOU bridge matches the empirical variance (dashed orange) of our trained models.
    The transparent black lines are the sampled approximate posterior paths used to calculate the empirical variance.\vspace{-3mm}
    }
    \label{fig:bridge}
\end{figure}

\section{Experiments}\label{sec:exp}\vspace{-1mm}
We implemented our method in JAX~\citep{jax2018github}, using Diffrax~\citep{kidger2021on} for
SDE solvers, Optax~\citep{deepmind2020jax} for optimization, Diffrax~\citep{deepmind2020jax} for distributions and Flax~\citep{flax2020github} for neural networks.
Unlike~\cite{tong2022learning} our approach is agnostic to discretization and the choice of the solver. Hence, in all experiments we can use the \emph{Stratonovich--Milstein} solver,
\cf~\cref{app:models} for more details.

\paragraph{Recovering the fractional Ornstein--Uhlenbeck bridge}
Applying our method on linear problems, allows comparing empirical results to
analytical formulations derived \eg using Gaussian process methodology~\cite{rasmussen2006gaussian}.
We begin by assessing the reconstruction capability of our method on a fractional Ornstein--Uhlenbeck (fOU) bridge, that is an OU--process driven by fBM: $\diff X(t)=-\theta X(t) \dt+\diff B_H$, starting at $X(0) = 0$ and conditioned to end at $X(T) = 0$.
Following the rules of Gaussian process regression~\citep[Eq. 2.24]{rasmussen2006gaussian}, we have an analytical expression for the posterior covariance:
\begin{equation}
    \label{eq:posterior-bridge}
    \expE{\tilde{X}(t)^2} = K(t, t) - 
    \begin{bmatrix}
        K(t,0)&K(t,T)
    \end{bmatrix}
    \begin{bmatrix}
        K(0,0)&K(T,0)\\K(0,T)&K(T,T)+\sigma^2
    \end{bmatrix}^{-1}
    \begin{bmatrix}
        K(0,t) \\ K(T,t)
    \end{bmatrix}
\end{equation}
where $K(t,\tau)$ is the prior kernel and the observation noise is $0$ for $X(0)$ and $\sigma$ for $X(T)$.
If $\theta=0$, $K(t,\tau) = \E\left[B_H(t)B_H(\tau)\right]$ (\cref{eq:cov-type1}) and if $\theta > 0$ and $H > 1/2$, the kernel admits the following form~\citep[Appendix A]{lysy2013statistical}:
\begin{align}
    \label{eq:lysy}
    K(t,\tau)=\frac{(2H^2-H)}{2 \theta} \Biggl(e^{-\theta|t-\tau|}\left[\frac{ \Gamma(2 H-1)+\Gamma(2 H-1,|t-\tau|)}{\theta^{2 H-1}}
    + \int_0^{|t-\tau|} e^{\theta u} u^{2 H-2} \mathrm{~d} u\right]\Biggr)
\end{align}
where $\Gamma(z,x)=\int_x^\infty t^{z-1}e{-t} \dt$ is the upper incomplete Gamma function.
This allows us to compare the true posterior variance with the empirical variance of a model that is trained by maximizing the ELBO. %
for a data point $X(T)=0$. As this is equivalent to the analytical result (\cref{eq:posterior-bridge}),
we can compare the variances over time. As plotted in~\cref{fig:bridge}, for various $H$ and $\theta$ values, our VI can correctly recover the posterior variance,
\cf~\cref{app:experiments} for additional results.

\begin{wrapfigure}[11]{r}{0.3\linewidth}
\vspace{-6mm}
    \centering
    \includegraphics[width=\linewidth]{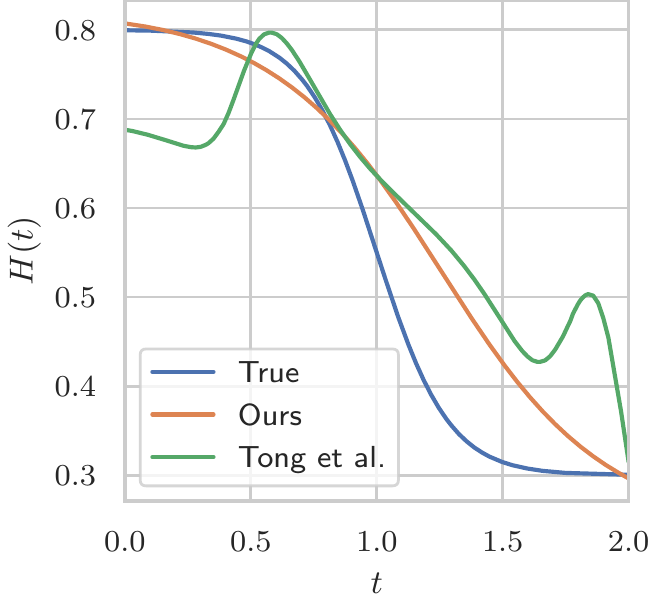}\vspace{-3mm}
    \caption{Estimating time--dependent $H(t)$ from data.}
    \label{fig:mbm}
\end{wrapfigure}
\vspace{1mm}\noindent\textbf{Estimating time-dependent Hurst index.\,}
Since our method of optimizing $\omega_k$ is tractable and differentiable,
we can directly optimize a parameterized $H$ by maximizing the ELBO.
Also a time--dependent Hurst index $H(t)$ can be modelled, leading to multifractional Brownian Motion~\citep{peltier1995multifractional}.
We directly compare with a toy problem presented in \cite[Sec. 5.2]{tong2022learning}.
We use the same model for $H(t)$, a neural network with one hidden layer of $10$ neurons and activation function tanh, and a final sigmoid activation, and the same input $[\sin(t), \cos(t), t]$.
We use $\hat{B}_H^{(II)}$ since their method is Type II.
\Cref{fig:mbm} shows a reasonable estimation of $H(t)$, which is more accurate than
the result from~\cite{tong2022learning}, \cf\cref{app:models} for more details.

\begin{figure}
    \centering
    \includegraphics[width=\textwidth, trim=20 0 80 0, clip]{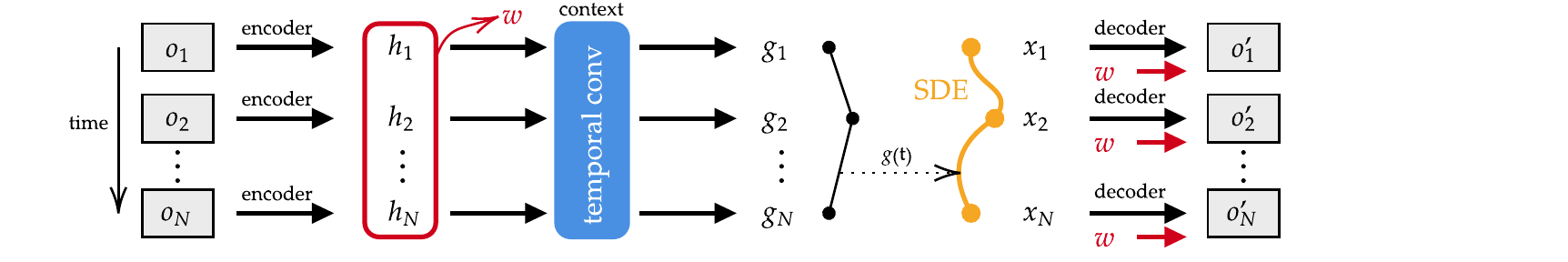}
    \caption{Schematic of the latent SDE video model.
    Video frames $\{o_i\}_i$ are encoded to vectors $\{h_i\}_i$.
    The static content vector $w$, that is free of the dynamic information, is inferred from $\{h_i\}_i$.
    The context model processes the information with temporal convolution layers, so that its outputs $\{g_i\}_i$ contain information from neighbouring frames.
    A linear interpolation on $\{g_i\}_i$ allows the posterior SDE model to receive time--appropriate information $g(t)$, at (intermediate) time--steps chosen by the SDE solver.
    Finally, the states $\{x_i\}_i$ and static $w$ are decoded to reconstruct frames $\{o'_i\}_i$. \vspace{-3mm}
    }
    \label{fig:video-scheme}
\end{figure}

\vspace{1mm}\noindent\textbf{Latent video models}
We apply variatonal inference for MA-fBMSDE on latent neural-SDE video modelling.
See \cref{fig:video-scheme} for a schematic explanation of our model.
We refer to \cref{app:models} for a detailed 

\begin{wraptable}[8]{r}{46mm}
\small
    \caption{Mean PSNR in video prediction. \vspace{-2mm}}
  \label{table:results}
  \centering
  \begin{tabular}{l|cc}
    Model & ELBO & PSNR \\
    \midrule
    SVG & N/A & $14.50$ \\
    SLRVP & N/A & ${16.93}$ \\
    BM & ${-913.60}$ & $14.90$ \\
    MA-fBM & $-608.00$ & $15.30$ \\
  \end{tabular}
\end{wraptable}
explanation of submodel architectures and hyperparameters.
Our latent video model is trained on Stochastic Moving MNIST~\citep{denton2018stochastic},
a video dataset where two MNIST numbers move on a canvas and bounce off the edge in random directions. We compare our model driven by MA-fBM to a baseline model driven by BM. For the baseline we set $K=1$, $\gamma_1=0$ and $\omega_1$ = 1, which is identical to white Brownian motion.
This allows us to study only the impact of the driving noise, uneffected by other desig choices.

As shown in~\cref{table:results}, our MA-fBM driven model is on par with closely related discrete-time methods such as SVG~\citep{denton2018stochastic} or SLRVP~\cite{franceschi2020stochastic}, in terms of PSNR. The Hurst index was optimized during training, and reached
$H=0.90$ at convergence (long-term memory).
The MA-fBM model achieves higher ELBO and PSNR on the test set compared to the BM version, indicating the added degree of freedom of the Hurst index benefits the model, and MA-fBM with $H=0.90$ is better suited to the data than BM.~\Cref{fig:posterior-smmnist} shows reconstructed posterior samples for the BM and MA-fBM models, conditioned on the same data.
Using the learned prior SDE as a generative model,~\cref{fig:prior-smmnist} reveals that our MA-fBMSDE better captures the data diversity.
\begin{figure}[h]
    \tiny
    \raggedleft
    Ground truth \quad
    $\vcenter{\hbox{\includegraphics[width=0.90\textwidth,trim=0 0 0 0,clip] {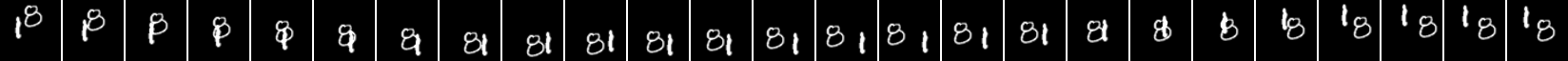}}}$
    BM \quad
    $\vcenter{\hbox{\includegraphics[width=0.90\textwidth,trim=0 0 0 0,clip]{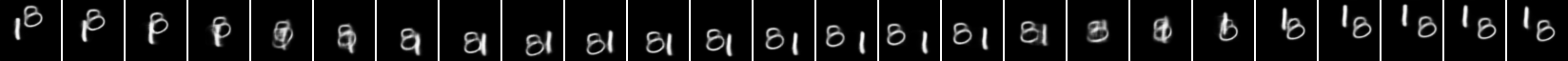}}}$
    \linebreak MA-fBM \quad
    $\vcenter{\hbox{\includegraphics[width=0.90\textwidth,trim=0 0 0 0,clip]{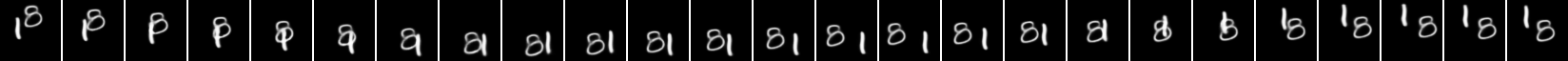}}}$
    \caption{Posterior reconstructions of a model driven by BM and a model driven by MA-fBM, conditioned on the same data ('Ground truth').\vspace{-2mm}
    }
    \label{fig:posterior-smmnist}
\end{figure}
\begin{figure}[h!]
    \tiny
    \raggedleft
    BM (1) \quad
    $\vcenter{\hbox{\includegraphics[width=0.90\textwidth]{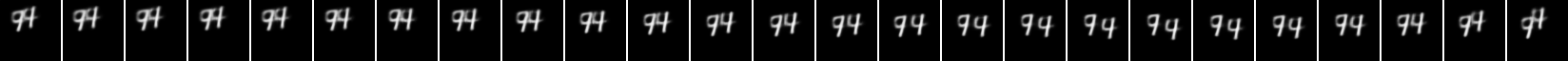}}}$
    \linebreak BM (2) \quad
    $\vcenter{\hbox{\includegraphics[width=0.90\textwidth]{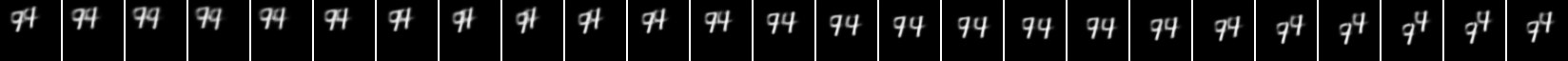}}}$
    \linebreak BM (3) \quad
    $\vcenter{\hbox{\includegraphics[width=0.90\textwidth]{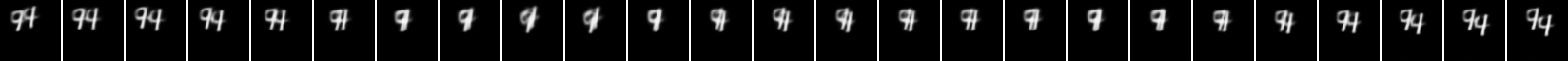}}}$
    \linebreak BM (4) \quad
    $\vcenter{\hbox{\includegraphics[width=0.90\textwidth]{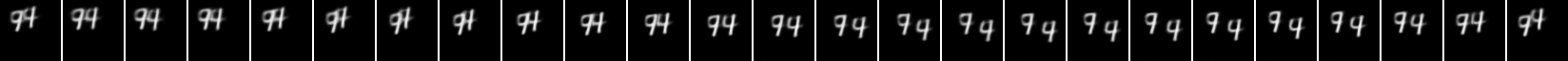}}}$
    \linebreak MA-fBM (1) \quad
    $\vcenter{\hbox{\includegraphics[width=0.90\textwidth]{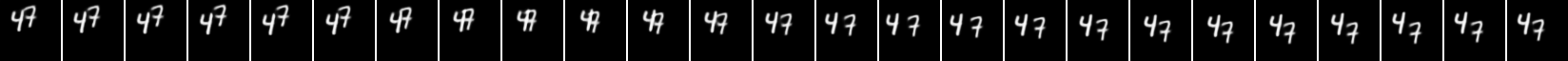}}}$
    \linebreak MA-fBM (2) \quad
    $\vcenter{\hbox{\includegraphics[width=0.90\textwidth]{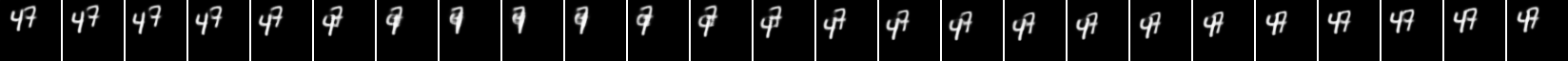}}}$
    \linebreak MA-fBM (3) \quad
    $\vcenter{\hbox{\includegraphics[width=0.90\textwidth]{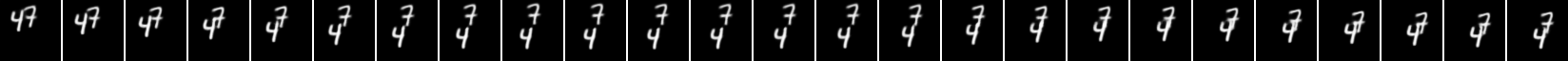}}}$
    \linebreak MA-fBM (4) \quad
    $\vcenter{\hbox{\includegraphics[width=0.90\textwidth]{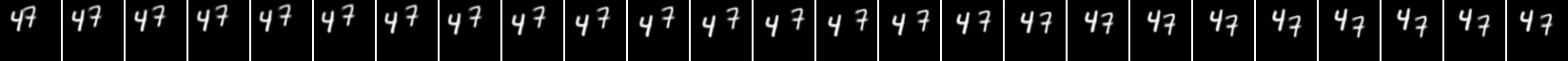}}}$
    \caption{Stochastic predictions using the trained prior of a model driven by BM and a model driven by MA-fBM, where the initial state is conditioned on the same data. Four samples are shown for each model.
    The MA-fBM samples show more diverse movements, thus better capturing the dynamics in the data. The BM samples are more similar, indicating a less powerful prior was learned.
    }
    \label{fig:prior-smmnist}
\end{figure}

\newpage
\subsection{Ablations \& Further Studies}
\,\\
\vspace{-15mm}
\begin{wrapfigure}[10]{r}{0.3\linewidth}
    \centering%
    \includegraphics[width=\linewidth]{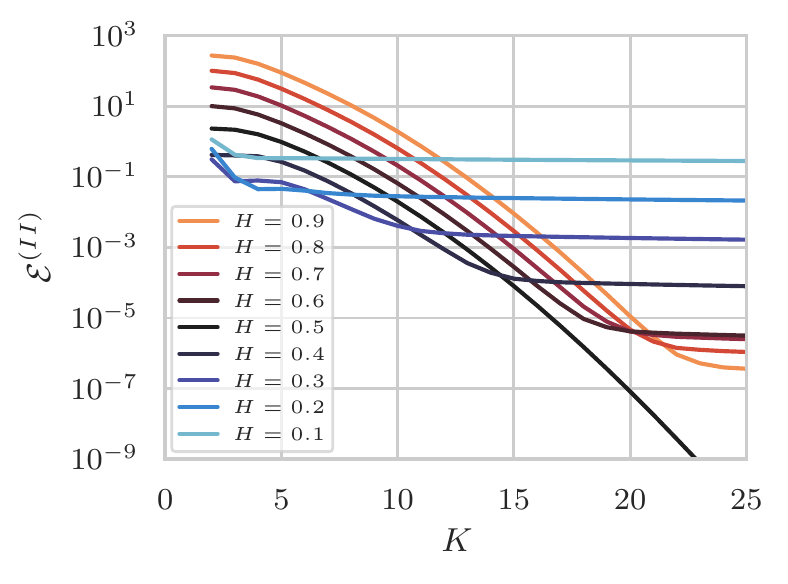}\vspace{-3mm}
    \caption{${\cal{E}}^{(II)}$ vs. $K$.}
    \label{fig:k-vs-error}
\end{wrapfigure}

\noindent\vspace{1mm}\textbf{Numerical study of the Markov approximation.\,}
By numerically evaluating the criterion ${\cal{E}}^{(II)}$ we can investigate the effect of $K$, the number of OU--processes, on the quality of the approximation.
\Cref{fig:k-vs-error} indicates that the approximation error diminishes by
increasing $K$.
However, after a certain threshold the criterion saturates, depending on $H$.
Adding more processes, especially for low $H$ brings diminishing returns.
The rapid convergence evidenced in this empirical result well agrees with the theoretical findings of~\citep{bayer2023markovian} especially for the \emph{rough processes} where $H<1/2$, as recalled in~\cref{thm:alfonsi}.

\begin{wrapfigure}[12]{r}{0.3\linewidth}
    \centering
    \vspace{-4mm}
    \includegraphics[width=\linewidth]{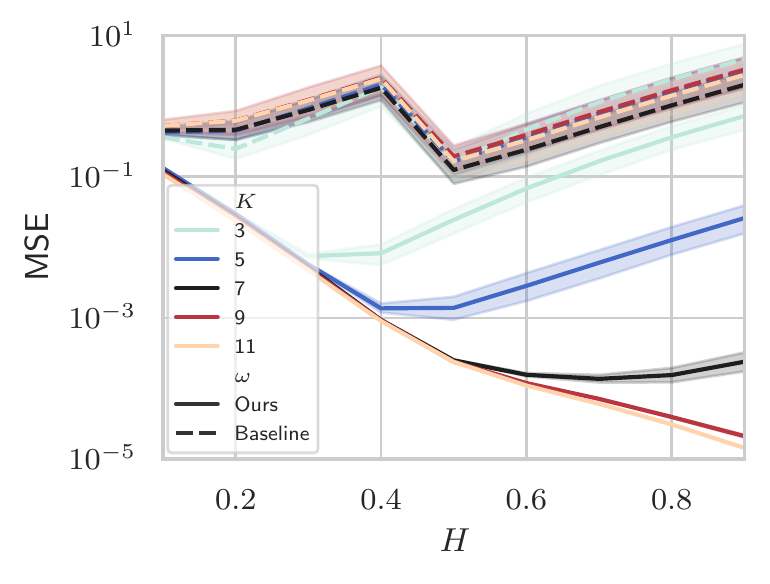}\vspace{-3mm}
    \caption{Mean square error (MSE) with $95\%$ confidence intervals vs. $H$ for varying $K$.}
    \label{fig:val-type2}
\end{wrapfigure}
\vspace{1mm}\noindent\textbf{MSE of the generated trajectories for MA-fBM and for varying $K$.\,}
On a more practical level, we take integration and numerical errors into account by simulating paths using MA-fBM and comparing to paths of the true integral
driven by the \emph{same} Wiener noise.
This is only possible for Type II, as for Type I one would need to start the integration from $-\infty$. Paths are generated from $t=0$ to $t=10$, with $4000$ integration steps for the approximation and $40000$ for the true integral.
We generate the paths over a range of Hurst indices and different $K$ values.
For each setting, $16$ paths are sampled.
Our approach for optimising $\omega_k$ values (\cref{sec:omega}) is compared to a baseline where $\omega_k$ is derived by a piece-wise approximation of the Laplace integral (\cf~\cref{app:gamma-omega-riemann}).~\cref{fig:val-type2} shows considerably better results in favor of our approach. Increasing $K$ has a rapid positive impact on the accuracy of the approximation with diminishing returns, further confirming our theoretical insights in~\cref{sec:bg}. We provide examples of individual trajectories generated in this experiment in~\cref{app:trajectories}.

\begin{wrapfigure}[9]{r}{0.3\linewidth}
    \centering
    \vspace{-5mm}
    \includegraphics[width=\linewidth]{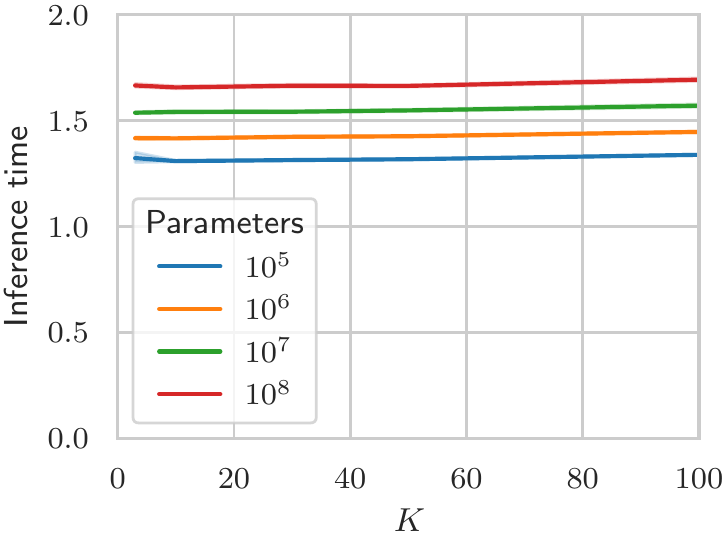}\vspace{-3mm}
    \caption{ $K$ vs. the run-time.}
    \label{fig:k-vs-time}
\end{wrapfigure}
\vspace{1mm}\noindent\textbf{Impact of $K$ and the $\#$parameters on inference time.\,}
We investigate the factors that influence the run-time in~\cref{fig:k-vs-time}, where $K$ OU--processes are gradually included to systems with increasing number of network parameters. Note that, since our approximation is driven by \emph{$1$ Wiener process}, and the control function $u(\postZ(t), t)$ is scalar, the impact on computational load of including more processes is limited and the run-time is still dominated by the size of the neural networks. 
This is good news as different applications might demand different number of OU--processes.

\section{Conclusion}
In this paper, we have proposed a new approach for performing variational inference on stochastic differential equations driven by \emph{fractional} Brownian motion (fBM). 
We began by uncovering the relatively unexplored Markov representation of fBM, allowing us to approximate non-Markovian paths using a linear combination of Wiener processes. This approximation enabled us to derive evidence lower bounds through Girsanov's change of measure, yielding posterior path measures as well as likelihood estimates. We also solved for optimal coefficients for combining these processes, in closed form.
Our diverse experimental study, spanning fOU bridges and Hurst index estimation, have consistently validated the effectiveness of our approach.
Moreover, our novel, continuous-time architecture, powered by Markov-approximate fBM driven neural-SDEs, has demonstrated improvements in video prediction, particularly when inferring the Hurst parameter during inference.

\paragraph{Limitations and future work} 
In our experiments, we observed increased computational overhead for larger time horizons due to SDE integration, although the expansion of the number of processes incurred minimal runtime costs.
We have also observed super-polynomial convergence empirically and recalled weaker polynomial rates in the literature.
Our Markov approximation still lacks a tight convergence bound. 
Our future work will also extend our framework to (fractional) Levy processes, which offer enhanced capabilities for modeling \emph{heavy-tailed} noise/data distributions.

\paragraph{Acknowledgments}
The authors thank Jonas Degrave and Tom Lefebvre for insightful discussions. M.O. acknowledges funding by Deutsche Forschungsgemeinschaft (DFG)-SFB1294/ 1-318763901. This research received funding from the Flemish Government under the "Onderzoeksprogramma Artificiële Intelligentie (AI) Vlaanderen" programme. Furthermore it was supported by Flanders Make under the SBO project CADAIVISION. 

\section*{Ethics Statement}
Our work is driven by a dedication to the advancement of knowledge and the betterment of society. While being largely theoretical, similar to many works advancing artificial intelligence, our work deserves an ethical consideration, which we present below.

All of our experiments were either run on publicly available datasets or on data that is synthetically generated. No human or animal subjects have been involved at any stage of this work.  
Our models are designed to enhance the understanding and prediction of real-world processes without causing harm or perpetuating unjust biases, unless provided in the datasets. While we do not foresee any issue with methodological bias, we have not analyzed the inherent biases of our algorithm and there might be implications in applications demanding utmost fairness. 

We aptly acknowledge the contributions of researchers whose work laid the foundation for our own. Proper citations and credit are given to previous studies and authors. All authors declare that there are no conflicts of interest that could compromise the impartiality and objectivity of this research. All authors have reviewed and approved the final manuscript before submission.

\section*{Reproducibility Statement}
We are committed to transparency in research and for this reason will make our implementation publicly available upon publication. To demonstrate our dedication, we have submitted all source code as part of the appendix. Considerable parts involve: (i) the Markov approximation and optimisation of the $\omega_k$ coefficients; (ii) maximising ELBO to perform variational inference between the prior $\diff Z(t)$ and the posterior $\diff \hat{Z}(t)$ and (iii) the novel neural-SDE based video prediction architecture making use of all our contributions.

\appendix
\section*{Appendix}\vspace{-1mm}
\section{Further Discussions}\vspace{-1mm}
\paragraph{Optimal choices for $\omega$ and $\gamma$ values}
Regarding the Type II case, there are different ways of determining $\gamma_k$ and $\omega_k$ in the literature~\citep{carmona1998fractional,bayer2023markovian,harms2019affine} some of which can lead to super-polynomial convergence~\citep{bayer2023markovian} under certain assumptions, while more general choices are still shown to converge, though with a weaker rate~\citep{alfonsi2021approximation} while still being \emph{strong} (path-wise) and of arbitrarily high polynomial order~\citep{harms2020strong}.
Some of these works state that such geometric choice of the quadrature intervals simplifies the proofs while being not optimal and smarter choices can exist (even with better rate of convergence). This is the reason why we believe that our computationally tractable, closed form expressions which optimally solve for these values lead to good, super-polynomial convergence both for types II and I (since the first type also admits a similar type of analysis).

\paragraph{Practical considerations for choosing $\gamma_k$}
Defining $\gamma_k$ as $(1/\gamma_{\max},\ldots,\gamma_{\max})$ is
a convenient way to indicate some practical considerations for choosing $\gamma_k$.
\cite{carmona1998simultaneous} show that $\gamma \dt > 1/2$ leads to unstable integration of the OU--process, where $\dt$ is the integration step. Care should be taken that $\gamma_{\max} \dt < 1/2$, either by decreasing $\gamma_{\max}$ or decreasing the integration step $\dt$.
Additionally, choosing large values for $\gamma$ is undesirable for numerical reasons.
Especially when using lower precision, numerical overflow can be a problem.
Since an OU--process reaches equilibrium after time $1/\gamma$, a practical
lower bound for $\gamma_{\max}$ is the length of the modelled sequences.
This ensures that memory of the MA-fBM process is modelled for at least the length of the sequence.

\paragraph{Time horizon for optimising $\omega_k$}
The closed form expressions for $\omega_k$ are in function of $H$ and the
time horizon $T$ (\cref{eq:omega-bI}).
Since the criterion is defined over the time interval $[0, \; T]$, it makes
sense to choose $T$ equal to the typical (or maximal) length of sequences
in the modelled dataset.
Specifically for "Type I", we advise to choose $T$ at two or three times the modelled sequence length, as at $t=0$, this process is already at equilibrium,
and its 'history' should be accounted for in the criterion.
We have observed better empirical results when choosing $T$ at a multiple of the sequence length.

\paragraph{State dependent diffusions}
For the case, where the diffusion $\sigma(X,t)$  explicitly depends on the state variable $X$, our Markovian approximation results in a 'standard' white noise SDE for the augmented system. 
As such, it does not suffer from problems with proper definitions of stochastic integrals as compared to the original SDE driven by fBM for such cases. 
Hence, a straightforward Itô--interpretation of our augmented SDE is, in principle, possible.
This might indicate, at first glance, that simple numerical solvers such as Euler's method could be sufficient for simulating the augmented SDE required 
for computing posterior expectations for the ELBO. 
While this point needs further theoretical investigation, preliminary simulations for for simple models with state dependent diffusions indicate that an Euler approximation (in accordance with known results for direct simulations of SDE driven by fBM~\citep{lysy2013statistical}) quickly lead to deviations from known analytical results. 
Hence, for state dependent diffusions, we resort to the Stratonovich interpretation of the augmented system and use corresponding higher order solvers~\cite{kidger2021on}\footnote{see \eg {\url{https://docs.kidger.site/diffrax/usage/how-to-choose-a-solver/\#stochastic-differential-equations}}}.
This approach yields excellent (pathwise) agreements with exact analytical results as we show in~\cref{sec:exp}. Although the ELBO for SDE is derived from Girsanov's change of measure theorem for Itô--SDE, by the known correspondence (resulting in a change of drift functions, when diffusions are state dependent)~\citep{gardiner1985handbook} between Itô and Stratonovich SDE
we conclude that within this approach, optimisation of the ELBO with respect to model parameters will also yield the corresponding estimates for the Stratonovich interpretation.

\paragraph{On initial values for "Type I"}
The initial values for "Type I" can be understood as resulting from an OU--process which was started at some negative time $t \to -\infty$ so that
\begin{equation}
    Y^{(I)}_k(0) = \int_{-\infty}^0 e^{\gamma_k s} \dWs
    \label{eq:y0-integral}
\end{equation}
and $Y^{(I)}_k(0)$ can be considered as samples from the joint stationary distribution. 
Because the stationary distribution is normal~\cite[Theorem 2.16]{harms2019affine} we can simply sample initial states of the $Y_k(t)$ processes for Type I with covariance $\expE{Y_i(0)Y_j(0)}$.
Using Itô isometry~\citep{oksendal2003stochastic}:
\begin{align}
    \expE{Y_i(0)Y_j(0)} &= \expE{\int_{-\infty}^0 e^{\gamma_i s} \dWs \int_{-\infty}^0 e^{\gamma_j s} \dWs} \\
    &= \int_{-\infty}^0 e^{(\gamma_i + \gamma_j) s} \diff s \\
    &= \frac{1}{\gamma_i + \gamma_j} \, .
    \label{eq:cov-y0}
\end{align}

\section{Proofs and Further Theoretical Details}
\subsection{The Girsanov theorem II and the KL divergence of measures}
\label{app:girsanov2}
We now state the variation II of the Girsanov theorem~\citep{oksendal2003stochastic} in our notation.
Let $X(t)\in\R^n$ be an It\^{o} process w.r.t. measure $P$ of the form:
\begin{equation}
    \dXt = b_{\theta}\left(X(t), t\right) \dt + \sigma_{\theta}\left(X(t), t\right) \dWt,
\end{equation}
where $0\leq t\leq T$, $W(t)\in\R^m$, $b_{\theta}\left(X(t), t\right)\in\R^n$ and $\sigma_{\theta}\left(X(t), t\right)\in\R^{n\times m}$. Define a measure $Q$ via:
\begin{equation}
\label{eq:Radon1}
    \frac{dQ}{dP} = M_T\vcentcolon=\exp\left[- \int_{0}^T u(X(t),t)\dWt - \frac{1}{2} \int_{0}^T u^2(X(t),t)\dt\right].
\end{equation}
Then
\begin{equation}
    W^\prime (t) := \int_0^T u(X(t), t) dt + W(T) 
\end{equation}
is a Brownian motion w.r.t. $Q$ and the process $X(t)$ has the following representation in terms of $B^\prime(t)$:
\begin{equation}
    \dXt = \alpha_{\theta}\left(X(t), t\right) \dt + \sigma_{\theta}\left(X(t), t\right) \diff{{W^\prime}(t)},
\end{equation}
where the new drift is:
\begin{equation}
\alpha_{\theta}\left(X(t), t\right) = b_{\theta}\left(X(t), t\right) -\sigma_{\theta}\left(X(t), t\right)u\left(X(t), t\right).
\end{equation}
We can also rewrite the Radon--Nykodim derivative in~\cref{eq:Radon1} as
\begin{align}
\label{Radon2}
\frac{dQ}{dP} &= \exp\left[ \int_0^T u\left(X(t), t\right) \dWt  - \frac{1}{2}  \int_0^T u^2\left(X(t), t\right)\dt \right]  \\
&=\exp\left[ \int_0^T u\left(X(t), t\right)  \left(\diff{W^{\prime}(t)} + u\left(X(t), t\right)\dt\right) - \frac{1}{2}  \int_0^T u\left(X(t), t\right) \dt \right]  \\
&=\exp\left[ \int_0^T u\left(X(t), t\right) \diff{W^{\prime}(t)} + \frac{1}{2}  \int_0^T u\left(X(t), t\right) \dt \right].
\end{align}
Thus, similar to~\cite{li2020scalable}, we get the KL divergence
\begin{equation}
E_Q\left[\ln \frac{dQ}{dP} \right] = \frac{1}{2}  \int_0^T E_Q[u^2\left(X(t), t\right)] dt.
\end{equation}

\section{Covariances}
The full derivation of covariances between some processes relevant to this work are described here. 

\paragraph{Fractional Brownian motion (Type II)}
Using Itô isometry~\citep{oksendal2003stochastic} we know that for $t>s$
\begin{equation}
\E\left[\int_0^t (t - u)^{H-1/2} dW_u \int_0^s (s-u)^{H-1/2} dW_u \right] = \int_0^s \left((t-u)(s-u)\right)^{H-1/2} \diff u
\end{equation}
Thus
\begin{equation}
\E\left[B_H^{(II)}(t) B_H^{(II)}(s)\right] = \frac{1}{\Gamma^2(H+1/2)} \int_0^s \left((t-u)(s-u)\right)^{H-1/2} \diff u
\label{eq:cov-b2}
\end{equation}

\paragraph{OU--processes driven by the same Wiener process}
Observe two Ornstein--Uhlenbeck processes driven by the same Wiener process:
\begin{equation}
    \begin{dcases}
        dY_i(t) = - \gamma_i Y_i(t) dt + dW(t) \\
        dY_j(t) = - \gamma_j Y_j(t) dt + dW(t)
    \end{dcases}
\end{equation}
Their covariance can be written as:
\begin{align}
    \Cov(Y_i(t), Y_j(t)) &= \E\left[(Y_i(t) - \E\left[Y_i(t)\right]) (Y_j(t) - \E\left[Y_j(t)\right])\right] \\
     &= \E\left[Y_i(t) Y_j(t)\right] \\
     &= \E\left[\int_0^t e^{- \gamma_i (t - s)} \diff W(s) \int_0^t e^{- \gamma_j (t - s)} \diff W(s)\right]\\
     \label{eq:cov-yyij-uses-ito}
     &= \int_0^t e^{- (\gamma_i + \gamma_j) (t - s)} \diff s \\
    \label{eq:cov-yyij}
    &= \frac{1}{\gamma_i + \gamma_j} - \frac{e^{- (\gamma_i + \gamma_j) t}}{\gamma_i + \gamma_j}
\end{align}
where \cref{eq:cov-yyij-uses-ito} is obtained following the Itô isometry~\citep{oksendal2003stochastic}.

\paragraph{Markov approximated fractional Brownian motion (Type I)}
Recall that (\cref{remark:markov}) $$
\hat{B}_H^{(I)}(t) = \sum_k \omega_k ( Y_k(t) - Y_k(0))
$$
where (\cref{eq:solved-y-sde})
$$Y_k(t) - Y_k(0) = Y_k(0) ( e^{- \gamma_k t} - 1) + \int_0^t e^{- \gamma_k (t - s)} \diff W(s)
$$
and $\E[Y_i(0) Y_j(0)] = \frac{1}{\gamma_i + \gamma_j}$ (\cref{eq:cov-y0}).
For $t > \tau$:

\begin{align}
    \expE{\hat{B}_H^{(I)}(t) \hat{B}_H^{(I)}(\tau)} &= \E \left[\left(\sum_k \omega_k \left(Y_k(t) - Y_k(0)\right)\right)\left(\sum_k \omega_k \left(Y_k(\tau) - Y_k(0)\right)\right) \right] \\
     &= \sum_{i,j} \omega_i \omega_j \E[\left(Y_i(t) - Y_i(0)\right) \left(Y_j(\tau) - Y_j(0)\right)] \\
     &= \sum_{i,j} \omega_i \omega_j \E\left[\left(Y_i(0) (e^{- \gamma_i t} - 1) + \int_0^t e^{- \gamma_i (t - s)} \diff W(s)\right) \right. \\ &\qquad\qquad\quad\left. \cdot\left(Y_j(0) (e^{- \gamma_j \tau} - 1) + \int_0^\tau e^{- \gamma_j (\tau - s)} \diff W(s) \right) \right] \nonumber \\
    &= \sum_{i,j} \omega_i \omega_j \biggl( \expE{Y_i(0) Y_j(0)}(e^{- \gamma_i t} - 1)(e^{- \gamma_j \tau} - 1) \\
    &\qquad\qquad\qquad\qquad + \int_0^\tau ( e^{- \gamma_i (t - s)}e^{- \gamma_j (\tau - s)} \diff s \biggr) \nonumber \\
     &= \sum_{i,j} \omega_i \omega_j \frac{1 - e^{- \gamma_i t} - e^{- \gamma_j \tau} + e^{- \gamma_i (t - \tau)}}{\gamma_i + \gamma_j}
     \label{eq:cov-bh-1}
\end{align}

\paragraph{Markov approximated fractional Brownian motion (Type II)}
Recall that (\cref{remark:markov})
$$
\hat{B}_H^{(II)}(t) = \sum_k \omega_k Y_k(t), \qquad Y_k(0)=0, \quad k=1,\ldots,K
$$
and for $t > \tau$:
\begin{align}
    \expE{\hat{B}_H^{(II)}(t) \hat{B}_H^{(II)}(\tau)} &= \E \left[\left(\sum_k \omega_k Y_k(t)\right)\left(\sum_{k=1}^K \omega_k Y_k(\tau)\right) \right] \\
    &= \sum_{i,j} \omega_i \omega_j \E[Y_i(t) Y_j(\tau)] \\
    &= \sum_{i,j} \omega_i \omega_j \E\left[\int_0^t e^{- \gamma_i (t - s)} \diff W(s) \int_0^\tau e^{- \gamma_j (\tau - s)} \diff W(s) \right] \\
    &= \sum_{i,j} \omega_i \omega_j \int_0^\tau e^{- \gamma_i (t - s) - \gamma_j (\tau - s)} ds \\
    &= \sum_{i,j} \omega_i \omega_j \left( \frac{e^{- \gamma_i (t - \tau)}}{\gamma_i + \gamma_j} - \frac{e^{- \gamma_i t - \gamma_j \tau}}{\gamma_i + \gamma_j} \right)
    \label{eq:cov-bh-2}
\end{align}

\paragraph{fBM and MA-fBM (Type I)}
Since (\cref{remark:markov}) $$
\hat{B}_H^{(I)}(t) = \sum_k \omega_k ( Y_k(t) - Y_k(0))
$$
where (\cref{eq:solved-y-sde})
$$Y_k(t) - Y_k(0) = Y_k(0) ( e^{- \gamma_k t} - 1) + \int_0^t e^{- \gamma_k (t - s)} \diff W(s)
$$
and (\cref{eq:y0-integral})
$$
Y_k(0) = \int_{-\infty}^0 e^{\gamma_k s} \dWs \, .
$$
we can write
\begin{equation}
    \hat{B}_H^{(I)}(t) = \sum_k \omega_k \left( (e^{- \gamma_k t} - 1)\int_{-\infty}^0 e^{\gamma_k s} \dWs +  \int_0^t e^{- \gamma_k (t - s)} \diff W(s) \right) \, .
\end{equation}
This leads to the following derivation (using Itô isometry~\citep{oksendal2003stochastic}):
\begin{align}
        \expE{\hat{B}_H^{(I)}(t)B_H^{(I)}(t)} &= \frac{1}{\Gamma(H+1/2)}\sum_k \omega_k \E
        \Biggl[
        \biggl(
            \int_{-\infty}^0 \left((t-s)^{H-1/2}-(-s)^{H-1/2}\right) \dWs
            \nonumber \\
            &\qquad\qquad\qquad\qquad\qquad\qquad\qquad\qquad\qquad+ \int_0^t(t-s)^{H-1/2} \dWs
        \biggr)
        \nonumber \\
        &\qquad\qquad\cdot\left(
            (e^{- \gamma_k t} - 1)\int_{-\infty}^0e^{\gamma_k s}\dWs + \int_0^t e^{-\gamma_k(t-s)}\dWs
        \right)
        \Biggr] \\
        &= \frac{1}{\Gamma(H+1/2)}\sum_k \omega_k
        \Biggl(
            (e^{- \gamma_k t} - 1)\int_{-\infty}^0
                \left((t-s)^{H-1/2}-(-s)^{H-1/2}\right)e^{\gamma_k s}
            \diff s
            \nonumber \\ &\qquad\qquad\qquad\qquad\qquad\qquad\qquad\enskip+
            \int_0^t
                (t-s)^{H-1/2}e^{-\gamma_k(t-s)}
            \diff s
        \Biggr) \\
        &= \sum_k \omega_k \frac{2-e^{-\gamma_k t}-Q(H+1/2,\gamma_k t)e^{\gamma_k t}}{\gamma_k^{H+1/2}}
        \label{eq:cov-b-bh-1}
\end{align}

where $Q(z, x)=\frac{1}{\Gamma(z)}\int_x^\infty t^{z-1}e^{-t}\dt$ is the regularized upper incomplete gamma function.

\paragraph{fBM and MA-fBM (Type II)}
\begin{align}
        \expE{\hat{B}_H^{(II)}(t)B_H^{(II)}(t)} &= \frac{1}{\Gamma(H+1/2)}\sum_k \omega_k \expE{\int_0^t e^{-\gamma_k(t-s)}\dWs \int_0^t(t-s)^{H-1/2}\diff s} \\
        &= \frac{1}{\Gamma(H+1/2)}\sum_k \omega_k \int_0^t e^{-\gamma_k(t-s)}(t-s)^{H-1/2} \diff s \\
        &= \sum_k \omega_k \frac{P(H+1/2,\gamma_k t)}{\gamma_k^{H+1/2}}
        \label{eq:cov-b-bh-2}
\end{align}
where $P(z, x)=\frac{1}{\Gamma(z)}\int_0^x t^{z-1}e^{-t}\dt$ is the regularized lower incomplete gamma function.

\section{Choosing \texorpdfstring{$\omega_k$}~ values}
\subsection{Baseline}
\label{app:gamma-omega-riemann}
To approximate the integral in equation (\ref{eq:v}) for $H<1/2$ we do a piece-wise linear approximation of the integral between the known $Y_k(t)$ values:
\begin{equation}
    \sum_{k=1}^K \omega_k Y_k(t) = \sum_{k=1}^{K-1} \int_{\gamma_k}^{\gamma_{k+1}} \left( \frac{\gamma_{k+1} - \gamma}{\gamma_{k+1} - \gamma_k} Y_k(t) + \frac{\gamma - \gamma_k}{\gamma_{k+1} - \gamma_k} Y_{k+1}(t) \right) \mu(\gamma) \diff \gamma
\end{equation}
For $H>1/2$ we approximate $\partial_\gamma Y_\gamma(t)$ with finite differences:
\begin{equation}
    \sum_{k=1}^K \omega_k Y_k(t) = \sum_{k=1}^{K-1} - \frac{Y_{k+1}(t) - Y_k(t)}{\gamma_{k+1} - \gamma_k} \int_{\gamma_k}^{\gamma_{k+1}} \nu(\gamma) \diff \gamma
\end{equation}
This leads to the following proposal for $\omega_k$:
\begin{equation}
    \omega_k = 
    \begin{dcases}
            \frac{1}{\Gamma(\alpha) \Gamma(1 - \alpha)} \Biggl( \1_\mathrm{k>1} \frac{\frac{\gamma_k^{2-\alpha}-\gamma_{k-1}^{2-\alpha}}{2-\alpha} - \gamma_{k-1} \frac{\gamma_k^{1-\alpha}-\gamma_{k-1}^{1-\alpha}}{1-\alpha}}{\gamma_k-\gamma_{k-1}} \\
        \begin{aligned}
            + \1_\mathrm{k<K} \frac{\gamma_{k+1} \frac{\gamma_{k+1}^{1-\alpha} - \gamma_k^{1-\alpha}}{1-\alpha} - \frac{\gamma_{k+1}^{2-\alpha}-\gamma_k^{2-\alpha}}{2-\alpha}}{\gamma_{k+1}-\gamma_k} \Biggr),& \quad H<1/2
        \\
        \frac{- 1}{(2 - \alpha) \Gamma(\alpha) \Gamma(2 - \alpha)} \left( \1_\mathrm{k>1}\frac{\gamma_k^{2-\alpha} - \gamma_{k-1}^{2-\alpha}}{\gamma_k-\gamma_{k-1}} - \1_\mathrm{k<K}\frac{\gamma_{k+1}^{2-\alpha} - \gamma_k^{2-\alpha}}{\gamma_{k+1}-\gamma_k} \right),& \quad H>1/2
        \end{aligned}
    \end{dcases}
\end{equation}
where $\alpha = H + 1/2$.

\subsection{A Proof for the Optimized \texorpdfstring{$\omega_k$}~ Values}
\label{app:optimized-omega}
To optimize $\omega_k$ values, we first provide a closed form expression for the approximation error and then show how we can solve for the $\omega_k$ that minimize this error.

\paragraph{Type I} We will start by optimizing $\omega_k$ for Type I. Consider the error:
\begin{align}
{\cal{E}}^{(I)}(\bm{\omega}) &= \int_0^T \E\left[\left(\hat{B}^{(I)}_H(t) - \BHonet\right)^2\right] \\
&= \int_0^T \left(\expE{\hat{B}^{(I)}_H(t)^2} + \expE{\BHonet^2} - 2 \expE{\hat{B}^{(I)}_H(t) \BHonet} \right) \dt
\end{align}
Using \cref{eq:cov-type1,eq:cov-bh-1,eq:cov-b-bh-1}
\begin{align}
{\cal{E}}^{(I)}(\bm{\omega}) &= 
\int_0^T \bigg( \sum_{i, j} \omega_i \omega_j \frac{2-e^{-\gamma_i t}-e^{-\gamma_j t}}{\gamma_i+\gamma_j}+t^{2 H} \nonumber \\
&\quad-2 \sum_k \omega_k \frac{2-e^{-\gamma_k t}-Q\left(H+1 / 2, \gamma_k t\right) e^{\gamma_k t}}{\gamma_k^{H+1 / 2}} \bigg) \dt \\
&= \sum_{i, j} \omega_i \omega_j \frac{2 T+\frac{e^{-\gamma_i T}-1}{\gamma_i}+\frac{e^{-\gamma_j T}-1}{\gamma_j}}{\gamma_i+\gamma_j}+\frac{T^{2 H+1}}{2 H+1} \nonumber \\
&\quad-2 \sum_k \omega_k\left(\frac{2 T}{\gamma_k^{H+1 / 2}}-\frac{T^{H+1 / 2}}{\gamma_k \Gamma(H+3 / 2)}+\frac{e^{-\gamma_k T}-Q\left(H+1 / 2, \gamma_k T\right) e^{\gamma_k T}}{\gamma_k^{H+3 / 2}}\right)
\end{align}
This leads to the quadratic form ${\cal{E}}^{(I)}(\bm{\omega})=\bm{\omega}^T \mA^{(I)} \bm{\omega} - 2 {\vb^{(I)}}^T \bm{\omega} + c^{(I)}$ with
\begin{align}
    \mA^{(I)}_{i,j} &= \frac{2T + \frac{e^{-\gamma_i T}-1}{\gamma_i} + \frac{e^{-\gamma_j T}-1}{\gamma_j}}{\gamma_i + \gamma_j}\\
    \vb^{(I)}_k &=
        \frac{2T}{\gamma_k^{H+1/2}}
        - \frac{T^{H+1/2}}{\gamma_k \Gamma(H+3/2)}
        + \frac{e^{-\gamma_k T} - Q(H+1/2, \gamma_k T) e^{\gamma_k T}}{\gamma_k^{H+3/2}}\\
    c^{(I)} &= \frac{T^{2 H+1}}{2 H+1} \,.
\end{align}

\paragraph{Type II} We now repeat a similar procedure for the Type II.
\begin{align}
{\cal{E}}^{(II)}(\bm{\omega}) &= \int_0^T \E\left[\left(\hat{B}^{(II)}_H(t) - \BHtwot\right)^2\right] \\
&= \int_0^T \left(\expE{\hat{B}^{(II)}_H(t)^2} + \expE{\BHtwot^2} - 2 \expE{\hat{B}^{(II)}_H(t) \BHtwot} \right) \dt
\end{align}
Using \cref{eq:cov-type2,eq:cov-bh-2,eq:cov-b-bh-2}
\begin{align}
{\cal{E}}^{(II)}(\bm{\omega}) &= 
    \int_0^T \sum_{i, j} \omega_i \omega_j \frac{1-e^{-\left(\gamma_i+\gamma_j\right) t}}{\gamma_i+\gamma_j}+\frac{t^{2 H}}{2 H \Gamma(H+1 / 2)^2}-2 \sum_k \omega_k \frac{P\left(H+1 / 2, \gamma_k t\right)}{\gamma_k^{H+1 / 2}} \dt \\
    &= 
\sum_{i, j} \omega_i \omega_j \frac{T+\frac{e^{-\left(\gamma_i+\gamma_j\right) T}-1}{\gamma_i+\gamma_j}}{\gamma_i+\gamma_j}+\frac{T^{2 H+1}}{2 H(2 H+1) \Gamma(H+1 / 2)^2} \\
& -2 \sum_k \omega_k\left(\frac{T}{\gamma_k^{H+1 / 2}} P\left(H+1 / 2, \gamma_k T\right)-\frac{H+1 / 2}{\gamma_k^{H+3 / 2}} P\left(H+3 / 2, \gamma_k T\right)\right)
\end{align}
This leads to the quadratic form ${\cal{E}}^{(II)}(\bm{\omega})=\bm{\omega}^T \mA^{(II)} \bm{\omega} - 2 {\vb^{(II)}}^T \bm{\omega} + c^{(II)}$ with
\begin{align}
    \mA^{(II)}_{i,j} &= \frac{T + \frac{e^{-(\gamma_i+\gamma_j)T}-1}{\gamma_i + \gamma_j}}{\gamma_i + \gamma_j}\\
    \vb^{(II)}_k &=
        \frac{T}{\gamma_k^{H+1/2}} P(H + 1/2, \gamma_k T) - \frac{H+1/2}{\gamma_k^{H+3/2}} P(H+3/2, \gamma_k T)\\
    c^{(II)} &= \frac{T^{2 H+1}}{2 H(2 H+1) \Gamma(H+1 / 2)^2} \,.
\end{align}

\subsection{Numerically stable implementation of \texorpdfstring{$Q(z,x)e^{x}$}~}
The term $Q(H+1/2, \gamma_k T) e^{\gamma_k T}$ in \cref{eq:omega-bI} leads to numerical instability, since $\gamma_k T$ is typically a high number (for the highest $\gamma_k$).
On the other hand, $Q(H+1/2, \gamma_k T)$ is a low number for high $\gamma_k T$.
Our stable implementation makes use of a continued fraction~\cite[eq. (12.6.17)]{cuyt2008handbook}, using the 'Kettenbruch' notation~\cite[sec. 1.1]{cuyt2008handbook} for continued fractions:
\begin{align}
    Q(H+1/2, \gamma_k T) e^{\gamma_k T} &= \frac{\Gamma(H+1/2,\gamma_k T)}{\Gamma(H+1/2)}e^{\gamma_k T} \\
    &= \frac{1}{\Gamma(H+1/2)(\gamma_k T)^{H+1/2}} \underset{m=1}{\overset{\infty}{\mathrm K}}
    \left( \frac{a_m(H+1/2)/(\gamma_k T)}{1} \right)
\end{align}
where $a_m(a)$ is given by
\begin{equation}
    a_1(a)=1, \quad a_{2 j}(a)=j-a, \quad a_{2 j+1}(a)=j, \quad j \geq 1
\end{equation}
In practice we observe better accuracy with the original equation for $\gamma_k T < 10$,
where it is still stable, and only need $5$ fractions to approximate the equation for $\gamma_k T > 10$.

\section{Details on model architectures \& hyperparameters}
\label{app:models}
\subsection{fOU bridge}
For all experiments, $K=5$ and $\gamma_k = (\frac{1}{20},\ldots,20)$.
We use "Type I" and the optimal definitions for $\omega_k$, with a time horizon $T=6$.
The control function is a neural network with two hidden layers of each $1000$ neurons, with $\tanh$ activation function. Its input is represented as $[\sin{t}, \cos{t}, X(t), Y_1(t),\ldots,Y_K(t)]$.
The control function is initialized so that its output is $0$ at the start of training.
Models are trained for $2000$ training steps with a batch size of $32$.
We use the Adam~\citep{kingma2014adam} optimizer with fixed learning rate $10^{-3}$.
We use the \emph{Stratonovich--Milstein} SDE solver~\citep{kidger2021on} with an integration step of $0.01$.
The length of the bridge $T=2$ and observation noise $\sigma=0.1$.

\subsection{Time dependent Hurst index}
\label{app:mbm}
We directly compare our method with the data and estimate found
in the published codebase of~\cite{tong2022learning}\footnote{\url{https://github.com/anh-tong/fractional_neural_sde/blob/7565a2/fractional_neural_sde/example.ipynb}}.
We choose $K=5$ and $\gamma_k = (\frac{1}{20},\ldots,20)$ and
use "Type II" (to match the data and noise type in \cite{tong2022learning}).
The optimal definitions for $\omega_k$, with time horizon $T=2$ are used.
The control function is a neural network with two hidden layers of each $1000$ neurons, with $\tanh$ activation function. Its input is represented as $[\sin{t}, \cos{t}, \sin{2t}, \cos{2t},\ldots,\sin{5t},\cos{5t}, X(t), Y_1(t),\ldots,Y_K(t)]$.
The model is trained for $1000$ training steps with a batch size of $4$.
We use the Adam~\citep{kingma2014adam} optimizer with a learning rate $3\times10^{-3}$,
scheduled with cosine decay to $3\times10^{-4}$ by the end of training.
We use the \emph{Stratonovich-Milstein} SDE solver~\citep{kidger2021on}.
The integration step is $0.01$ and observation noise $\sigma=0.025$
(both identical to~\cite{tong2022learning}).

\subsection{Latent video model}%
For the MA-fBM model, $K=5$ and $\gamma_k = (\frac{1}{20},\ldots,20)$.
We use "Type I" and the corresponding definitions for $\omega_k$, with a time horizon $T=2.4$.
For the BM model, $K=1$, $\gamma_1=0$ and $\omega=1$, which naturally corresponds to white Brownian motion.
The number of latent dimensions $D=6$.

The encoder model consists of four blocks, containing a convolution layer, maxpool,
groupnorm and SiLU activation. Each block reduces spatial dimension by $2$, and the number of features in each block is $(64, 128, 256, 256)$.
The last output is flattened and is the input of a dense layer, with $h$ as output with $64$ features.

The median over the time axis of $h$ is fed into a two layers neural network to 
produce the static content vector $w$. Since the median is permutation invariant, 
$w$ contains no dynamic information, only static information.
$w$ also has $64$ features.

The context model consists of two subsequent $1-D$ convolutions in the temporal dimension. Thus, information is shared over different frames, which is necessary for inference. The output of this model is $g$.
Another model receives $(g_1, h_1, h_2, h_3)$ to infer $q_{x_1}$, the posterior distribution of the initial state of the SDE.
$x_1$ is sampled from $q_{x_1}$, which we model as a diagonal Normal distribution.
The prior $p_{x_1}$ is also optimized, and $\KL(p,q)$ is added to the loss function.

The prior drift $b_\theta(X,t)$ and the control function $u(Z(t),t)$ have the same architecture, a neural network with two hidden layers of each $200$ neurons, with $\tanh$ activation functions.
The shared diffusion $\sigma_\theta(X,t)$ is implemented so that the noise is commutative to allow Milstein solvers~\citep{li2020scalable,kidger2021efficient}, \ie $\sigma_\theta(X,t)$ is diagonal and the $i$-th component on the diagonal only receives $X_i(t)$ as input, where we have defined $D$ separate neural networks for each component.
Each neural network has two layers with $200$ neurons and $\tanh$ activations.

$b_\theta$ and $\sigma_\theta$ receive $X(t)$ as input.
The control function a concatenated vector of $(X(t), Y_1(t),\ldots,Y_K(t), g(t))$.
$g(t)$ is a linear interpolation of $g$ at time $t$.
This enables the control function to use appropriate information to be able
to steer the process correctly.

The resulting states $x$ after integration of the SDE are fed, together with the static
content vector $w$ in the decoder model.
The decoder model has first a dense layer.
The outputs of this first layer are shaped in a $4\times4$ spatial grad.
Subsequently, four blocks with a convolution layer, groupnorm, a spatial nearest neighbour upsampling layer and a SiLU activation.
Thus, the model reaches the correct resolution of $64\times64$.
Two additional convolution layers with SiLU activation and a final sigmoid activation
complete the decoder model.

We train on sequences of $25$ frames, with a time length of $2.4$ ($0.1$ per frame).
The frames have resolution $64\times64$ and $1$ color channel.
Each model was trained for $187500$ training steps with a batch size of $32$.
We use the Adam~\citep{kingma2014adam} optimizer with fixed learning rate $3\times10^{-4}$.
We use the \emph{Stratonovich--Milstein} SDE solver~\citep{kidger2021on} with an integration step of $0.033$ ($3$ integration steps per data frame).
Models were trained on a single NVIDIA GeForce RTX 4090, which takes around $39$ hours for $1$ model.

\section{Additional experimental results}
\label{app:experiments}
\subsection{Generated trajectories of MA-fBM for varying \texorpdfstring{$K$}~}
\label{app:trajectories}
Included here are some of the trajectories used to calculate the MSE
of the generated trajectories for MA-fBM for varying $K$ (\cref{fig:val-type2}).
We show trajectories of MA-fBM with our approach (\cref{sec:omega}) and the baseline method (\cf~\cref{app:gamma-omega-riemann}) for choosing $\omega_k$.
True paths are plotted in black, the approximations with varying $K$ in a color-scale as indicated in the legends,
see \cref{fig:trajectories-1,fig:trajectories-2,fig:trajectories-3,fig:trajectories-4,fig:trajectories-6,fig:trajectories-8,fig:trajectories-9}.
Our method quickly converges to the true path for increasing $K$,
while much slower for the baseline method.

\newcommand{\trajectorysubfigure}[1]{%
\begin{figure}[h!]
    \centering
    \begin{subfigure}{0.49\textwidth}
        \centering
        \includegraphics[width=\textwidth,trim=0 0 0 0,clip]{figures/trajectories/path-approximation-H=0.#1-baseline.pdf}
        \caption{Baseline}
    \end{subfigure}
    \begin{subfigure}{0.49\textwidth}
        \centering
        \includegraphics[width=\textwidth,trim=0 0 0 0,clip]{figures/trajectories/path-approximation-H=0.#1-ours.pdf}
        \caption{Ours}
    \end{subfigure}
    \caption{Generated trajectories of (\textbf{a}) the baseline method summarized in~\cref{app:gamma-omega-riemann} and (\textbf{b}) our method, MA-fBM (\cref{sec:omega}), 
    for varying $K$ and $H = 0.#1$.
    }
    \label{fig:trajectories-#1}
\end{figure}
}
\trajectorysubfigure{1}
\trajectorysubfigure{2}
\trajectorysubfigure{3}
\trajectorysubfigure{4}
\trajectorysubfigure{6}
\trajectorysubfigure{8}
\trajectorysubfigure{9}

\subsection{fOU Bridge}
\Cref{fig:bridge-extra} shows additional results of the fractional Ornstein--Uhlenbeck bridge.
The variances are calculated with \cref{eq:posterior-bridge}, and \cref{eq:lysy} for $\theta>0$ and $H>1/2$ or \cref{eq:cov-type1} for $\theta=0$.
Note that we do not have a useful covariance equation for $\theta>0$ and $H<1/2$~\citep{lysy2013statistical},
so this setting is not included in the experiments.

\newcommand{\fbmsubfigure}[2]{%
    \begin{subfigure}{0.32\textwidth}
        \centering
        \includegraphics[width=\textwidth,trim=6 8 7 7,clip]{figures/fbm-bridge/bridge_hurst_0.#10-theta_#2.00.pdf}
        \caption{$H = 0.#1, \theta=#2.0$}
    \end{subfigure}
}

\begin{figure}[h]
    \centering
    \fbmsubfigure{1}{0}
    \fbmsubfigure{2}{0}
    \fbmsubfigure{3}{0}
    \fbmsubfigure{4}{0}
    \fbmsubfigure{5}{0}
    \fbmsubfigure{6}{0}
    \fbmsubfigure{7}{0}
    \fbmsubfigure{8}{0}
    \fbmsubfigure{9}{0}
    \fbmsubfigure{6}{1}
    \fbmsubfigure{7}{1}
    \fbmsubfigure{8}{1}
    \fbmsubfigure{9}{1}
    \caption{The true variance (blue) of a fOU bridge matches the empirical variance (dashed orange) of our trained models.
    The transparent black lines are the sampled approximate posterior paths used to calculate the empirical variance.
    }
    \label{fig:bridge-extra}
\end{figure}


\begin{thebibliography}{51}
\providecommand{\natexlab}[1]{#1}
\providecommand{\url}[1]{\texttt{#1}}
\expandafter\ifx\csname urlstyle\endcsname\relax
  \providecommand{\doi}[1]{doi: #1}\else
  \providecommand{\doi}{doi: \begingroup \urlstyle{rm}\Url}\fi

\bibitem[Alfonsi \& Kebaier(2021)Alfonsi and Kebaier]{alfonsi2021approximation}
Aur{\'e}lien Alfonsi and Ahmed Kebaier.
\newblock Approximation of stochastic volterra equations with kernels of
  completely monotone type.
\newblock \emph{arXiv preprint arXiv:2102.13505}, 2021.

\bibitem[Ali et~al.(2023)Ali, Bond, Birdal, Ceylan, Karacan, Erdem, and
  Erdem]{ali2023vidstyleode}
Moayed~Haji Ali, Andrew Bond, Tolga Birdal, Duygu Ceylan, Levent Karacan, Erkut
  Erdem, and Aykut Erdem.
\newblock Vidstyleode: Disentangled video editing via stylegan and neuralodes.
\newblock In \emph{International Conference on Computer Vision (ICCV)}, 2023.

\bibitem[Allouche et~al.(2022)Allouche, Girard, and
  Gobet]{allouche2022generative}
Micha{\"e}l Allouche, St{\'e}phane Girard, and Emmanuel Gobet.
\newblock A generative model for fbm with deep relu neural networks.
\newblock \emph{Journal of Complexity}, 73:\penalty0 101667, 2022.

\bibitem[Babaeizadeh et~al.(2018)Babaeizadeh, Finn, Erhan, Campbell, and
  Levine]{babaeizadeh2018stochastic}
Mohammad Babaeizadeh, Chelsea Finn, Dumitru Erhan, Roy~H Campbell, and Sergey
  Levine.
\newblock Stochastic variational video prediction.
\newblock In \emph{International Conference on Learning Representations}, 2018.

\bibitem[Babuschkin et~al.(2020)Babuschkin, Baumli, Bell, Bhupatiraju, Bruce,
  Buchlovsky, Budden, Cai, Clark, Danihelka, Dedieu, Fantacci, Godwin, Jones,
  Hemsley, Hennigan, Hessel, Hou, Kapturowski, Keck, Kemaev, King, Kunesch,
  Martens, Merzic, Mikulik, Norman, Papamakarios, Quan, Ring, Ruiz, Sanchez,
  Sartran, Schneider, Sezener, Spencer, Srinivasan, Stanojevi\'{c}, Stokowiec,
  Wang, Zhou, and Viola]{deepmind2020jax}
Igor Babuschkin, Kate Baumli, Alison Bell, Surya Bhupatiraju, Jake Bruce, Peter
  Buchlovsky, David Budden, Trevor Cai, Aidan Clark, Ivo Danihelka, Antoine
  Dedieu, Claudio Fantacci, Jonathan Godwin, Chris Jones, Ross Hemsley, Tom
  Hennigan, Matteo Hessel, Shaobo Hou, Steven Kapturowski, Thomas Keck, Iurii
  Kemaev, Michael King, Markus Kunesch, Lena Martens, Hamza Merzic, Vladimir
  Mikulik, Tamara Norman, George Papamakarios, John Quan, Roman Ring, Francisco
  Ruiz, Alvaro Sanchez, Laurent Sartran, Rosalia Schneider, Eren Sezener,
  Stephen Spencer, Srivatsan Srinivasan, Milo\v{s} Stanojevi\'{c}, Wojciech
  Stokowiec, Luyu Wang, Guangyao Zhou, and Fabio Viola.
\newblock The {D}eep{M}ind {JAX} {E}cosystem, 2020.
\newblock URL \url{http://github.com/deepmind}.

\bibitem[Bayer \& Breneis(2023)Bayer and Breneis]{bayer2023markovian}
Christian Bayer and Simon Breneis.
\newblock Markovian approximations of stochastic volterra equations with the
  fractional kernel.
\newblock \emph{Quantitative Finance}, 23\penalty0 (1):\penalty0 53--70, 2023.

\bibitem[Bishop \& Nasrabadi(2006)Bishop and Nasrabadi]{bishop2006pattern}
Christopher~M Bishop and Nasser~M Nasrabadi.
\newblock \emph{Pattern recognition and machine learning}, volume~4.
\newblock Springer, 2006.

\bibitem[Bradbury et~al.(2018)Bradbury, Frostig, Hawkins, Johnson, Leary,
  Maclaurin, Necula, Paszke, Vander{P}las, Wanderman-{M}ilne, and
  Zhang]{jax2018github}
James Bradbury, Roy Frostig, Peter Hawkins, Matthew~James Johnson, Chris Leary,
  Dougal Maclaurin, George Necula, Adam Paszke, Jake Vander{P}las, Skye
  Wanderman-{M}ilne, and Qiao Zhang.
\newblock {JAX}: composable transformations of {P}ython+{N}um{P}y programs,
  2018.
\newblock URL \url{http://github.com/google/jax}.

\bibitem[Carmona \& Coutin(1998{\natexlab{a}})Carmona and
  Coutin]{carmona1998fractional}
Philippe Carmona and Laure Coutin.
\newblock Fractional brownian motion and the markov property.
\newblock \emph{Electronic Communications in Probability [electronic only]},
  3:\penalty0 95--107, 1998{\natexlab{a}}.

\bibitem[Carmona \& Coutin(1998{\natexlab{b}})Carmona and
  Coutin]{carmona1998simultaneous}
Philippe Carmona and Laure Coutin.
\newblock Simultaneous approximation of a family of (stochastic) differential
  equations.
\newblock \emph{Unpublished, June}, 915\penalty0 (10.1051), 1998{\natexlab{b}}.

\bibitem[Carmona et~al.(2000)Carmona, Coutin, and
  Montseny]{carmona2000approximation}
Philippe Carmona, Laure Coutin, and G{\'e}rard Montseny.
\newblock Approximation of some gaussian processes.
\newblock \emph{Statistical inference for stochastic processes}, 3:\penalty0
  161--171, 2000.

\bibitem[Cuyt et~al.(2008)Cuyt, Petersen, Verdonk, Waadeland, and
  Jones]{cuyt2008handbook}
Annie~AM Cuyt, Vigdis Petersen, Brigitte Verdonk, Haakon Waadeland, and
  William~B Jones.
\newblock \emph{Handbook of continued fractions for special functions}.
\newblock Springer Science \& Business Media, 2008.

\bibitem[Davidson \& Hashimzade(2009)Davidson and Hashimzade]{davidson2009type}
James Davidson and Nigar Hashimzade.
\newblock Type i and type ii fractional brownian motions: A reconsideration.
\newblock \emph{Computational Statistics \& Data Analysis}, 53\penalty0
  (6):\penalty0 2089--2106, 2009.

\bibitem[Denton \& Fergus(2018)Denton and Fergus]{denton2018stochastic}
Emily Denton and Rob Fergus.
\newblock Stochastic video generation with a learned prior.
\newblock In \emph{International conference on machine learning}, pp.\
  1174--1183. PMLR, 2018.

\bibitem[Franceschi et~al.(2020)Franceschi, Delasalles, Chen, Lamprier, and
  Gallinari]{franceschi2020stochastic}
Jean-Yves Franceschi, Edouard Delasalles, Micka{\"e}l Chen, Sylvain Lamprier,
  and Patrick Gallinari.
\newblock Stochastic latent residual video prediction.
\newblock In \emph{International Conference on Machine Learning}, pp.\
  3233--3246. PMLR, 2020.

\bibitem[Gardiner et~al.(1985)]{gardiner1985handbook}
Crispin~W Gardiner et~al.
\newblock \emph{Handbook of stochastic methods}, volume~3.
\newblock springer Berlin, 1985.

\bibitem[Gatheral et~al.(2018)Gatheral, Jaisson, and
  Rosenbaum]{gatheral2018volatility}
Jim Gatheral, Thibault Jaisson, and Mathieu Rosenbaum.
\newblock Volatility is rough.
\newblock \emph{Quantitative finance}, 18\penalty0 (6):\penalty0 933--949,
  2018.

\bibitem[Gojcic et~al.(2021)Gojcic, Litany, Wieser, Guibas, and
  Birdal]{gojcic2021weakly}
Zan Gojcic, Or~Litany, Andreas Wieser, Leonidas~J Guibas, and Tolga Birdal.
\newblock Weakly supervised learning of rigid 3d scene flow.
\newblock In \emph{Proceedings of the IEEE/CVF conference on computer vision
  and pattern recognition}, pp.\  5692--5703, 2021.

\bibitem[Gordon \& Parde(2021)Gordon and Parde]{gordon2021latent}
Cade Gordon and Natalie Parde.
\newblock Latent neural differential equations for video generation.
\newblock In \emph{NeurIPS 2020 Workshop on Pre-registration in Machine
  Learning}, pp.\  73--86. PMLR, 2021.

\bibitem[Guerra \& Nualart(2008)Guerra and Nualart]{guerra2008stochastic}
Joao Guerra and David Nualart.
\newblock Stochastic differential equations driven by fractional brownian
  motion and standard brownian motion.
\newblock \emph{Stochastic analysis and applications}, 26:\penalty0 1053--1075,
  2008.

\bibitem[Harms(2020)]{harms2020strong}
Philipp Harms.
\newblock Strong convergence rates for markovian representations of fractional
  processes.
\newblock \emph{Discrete and Continuous Dynamical Systems-B}, 26\penalty0
  (10):\penalty0 5567--5579, 2020.

\bibitem[Harms \& Stefanovits(2019)Harms and Stefanovits]{harms2019affine}
Philipp Harms and David Stefanovits.
\newblock Affine representations of fractional processes with applications in
  mathematical finance.
\newblock \emph{Stochastic Processes and their Applications}, 129:\penalty0
  1185--1228, 2019.

\bibitem[Hayashi \& Nakagawa(2022)Hayashi and Nakagawa]{hayashi2022fractional}
Kohei Hayashi and Kei Nakagawa.
\newblock Fractional sde-net: Generation of time series data with long-term
  memory.
\newblock In \emph{2022 IEEE 9th International Conference on Data Science and
  Advanced Analytics (DSAA)}, pp.\  1--10. IEEE, 2022.

\bibitem[Heek et~al.(2023)Heek, Levskaya, Oliver, Ritter, Rondepierre, Steiner,
  and van {Z}ee]{flax2020github}
Jonathan Heek, Anselm Levskaya, Avital Oliver, Marvin Ritter, Bertrand
  Rondepierre, Andreas Steiner, and Marc van {Z}ee.
\newblock {F}lax: A neural network library and ecosystem for {JAX}, 2023.
\newblock URL \url{http://github.com/google/flax}.

\bibitem[Ho et~al.(2022)Ho, Salimans, Gritsenko, Chan, Norouzi, and
  Fleet]{ho2022video}
Jonathan Ho, Tim Salimans, Alexey Gritsenko, William Chan, Mohammad Norouzi,
  and David~J Fleet.
\newblock Video diffusion models.
\newblock \emph{arXiv:2204.03458}, 2022.

\bibitem[Kappen \& Ruiz(2016)Kappen and Ruiz]{kappen2016adaptive}
Hilbert~Johan Kappen and Hans~Christian Ruiz.
\newblock Adaptive importance sampling for control and inference.
\newblock \emph{Journal of Statistical Physics}, 162:\penalty0 1244--1266,
  2016.

\bibitem[Kidger(2021)]{kidger2021on}
Patrick Kidger.
\newblock \emph{{O}n {N}eural {D}ifferential {E}quations}.
\newblock PhD thesis, University of Oxford, 2021.

\bibitem[Kidger et~al.(2021)Kidger, Foster, Li, and Lyons]{kidger2021efficient}
Patrick Kidger, James Foster, Xuechen~Chen Li, and Terry Lyons.
\newblock Efficient and accurate gradients for neural sdes.
\newblock \emph{Advances in Neural Information Processing Systems},
  34:\penalty0 18747--18761, 2021.

\bibitem[Kingma \& Ba(2014)Kingma and Ba]{kingma2014adam}
Diederik~P Kingma and Jimmy Ba.
\newblock Adam: A method for stochastic optimization.
\newblock \emph{arXiv preprint arXiv:1412.6980}, 2014.

\bibitem[Kong et~al.(2020)Kong, Sun, and Zhang]{kong2020sde}
Lingkai Kong, Jimeng Sun, and Chao Zhang.
\newblock Sde-net: equipping deep neural networks with uncertainty estimates.
\newblock In \emph{Proceedings of the 37th International Conference on Machine
  Learning}, pp.\  5405--5415, 2020.

\bibitem[Li et~al.(2020)Li, Wong, Chen, and Duvenaud]{li2020scalable}
Xuechen Li, Ting-Kam~Leonard Wong, Ricky~TQ Chen, and David~K Duvenaud.
\newblock Scalable gradients and variational inference for stochastic
  differential equations.
\newblock In \emph{Symposium on Advances in Approximate Bayesian Inference},
  pp.\  1--28. PMLR, 2020.

\bibitem[Liao et~al.(2019)Liao, Lyons, Yang, and Ni]{liao2019learning}
Shujian Liao, Terry Lyons, Weixin Yang, and Hao Ni.
\newblock Learning stochastic differential equations using rnn with log
  signature features.
\newblock \emph{arXiv preprint arXiv:1908.08286}, 2019.

\bibitem[Lim \& Sithi(1995)Lim and Sithi]{lim1995asymptotic}
SC~Lim and VM~Sithi.
\newblock Asymptotic properties of the fractional brownian motion of
  riemann-liouville type.
\newblock \emph{Physics Letters A}, 206\penalty0 (5-6):\penalty0 311--317,
  1995.

\bibitem[Liu et~al.(2019)Liu, Xiao, Si, Cao, Kumar, and Hsieh]{liu2019neural}
Xuanqing Liu, Tesi Xiao, Si~Si, Qin Cao, Sanjiv Kumar, and Cho-Jui Hsieh.
\newblock Neural sde: Stabilizing neural ode networks with stochastic noise.
\newblock \emph{arXiv preprint arXiv:1906.02355}, 2019.

\bibitem[Luo et~al.(2023)Luo, Chen, Zhang, Huang, Wang, Shen, Zhao, Zhou, and
  Tan]{luo2023videofusion}
Zhengxiong Luo, Dayou Chen, Yingya Zhang, Yan Huang, Liang Wang, Yujun Shen,
  Deli Zhao, Jingren Zhou, and Tieniu Tan.
\newblock Videofusion: Decomposed diffusion models for high-quality video
  generation.
\newblock In \emph{Proceedings of the IEEE/CVF Conference on Computer Vision
  and Pattern Recognition}, pp.\  10209--10218, 2023.

\bibitem[Lysy \& Pillai(2013)Lysy and Pillai]{lysy2013statistical}
Martin Lysy and Natesh~S Pillai.
\newblock Statistical inference for stochastic differential equations with
  memory.
\newblock \emph{arXiv preprint arXiv:1307.1164}, 2013.

\bibitem[Mandelbrot \& Van~Ness(1968)Mandelbrot and
  Van~Ness]{mandelbrot1968fractional}
Benoit~B Mandelbrot and John~W Van~Ness.
\newblock Fractional brownian motions, fractional noises and applications.
\newblock \emph{SIAM review}, 10\penalty0 (4):\penalty0 422--437, 1968.

\bibitem[Morrill et~al.(2021)Morrill, Salvi, Kidger, and
  Foster]{morrill2021neural}
James Morrill, Cristopher Salvi, Patrick Kidger, and James Foster.
\newblock Neural rough differential equations for long time series.
\newblock In \emph{International Conference on Machine Learning}, pp.\
  7829--7838. PMLR, 2021.

\bibitem[{\O}ksendal \& {\O}ksendal(2003){\O}ksendal and
  {\O}ksendal]{oksendal2003stochastic}
Bernt {\O}ksendal and Bernt {\O}ksendal.
\newblock \emph{Stochastic differential equations}.
\newblock Springer, 2003.

\bibitem[Opper(2019)]{opper2019variational}
Manfred Opper.
\newblock Variational inference for stochastic differential equations.
\newblock \emph{Annalen der Physik}, 531\penalty0 (3):\penalty0 1800233, 2019.

\bibitem[Park et~al.(2021)Park, Kim, Lee, Choo, Lee, Kim, and
  Choi]{park2021vid}
Sunghyun Park, Kangyeol Kim, Junsoo Lee, Jaegul Choo, Joonseok Lee, Sookyung
  Kim, and Edward Choi.
\newblock Vid-ode: Continuous-time video generation with neural ordinary
  differential equation.
\newblock In \emph{Proceedings of the AAAI Conference on Artificial
  Intelligence}, volume~35, 2021.

\bibitem[Peltier \& V{\'e}hel(1995)Peltier and
  V{\'e}hel]{peltier1995multifractional}
Romain-Fran{\c{c}}ois Peltier and Jacques~L{\'e}vy V{\'e}hel.
\newblock \emph{Multifractional Brownian motion: definition and preliminary
  results}.
\newblock PhD thesis, INRIA, 1995.

\bibitem[Rasmussen et~al.(2006)Rasmussen, Williams,
  et~al.]{rasmussen2006gaussian}
Carl~Edward Rasmussen, Christopher~KI Williams, et~al.
\newblock \emph{Gaussian processes for machine learning}, volume~1.
\newblock Springer, 2006.

\bibitem[Rempe et~al.(2020)Rempe, Birdal, Zhao, Gojcic, Sridhar, and
  Guibas]{rempe2020caspr}
Davis Rempe, Tolga Birdal, Yongheng Zhao, Zan Gojcic, Srinath Sridhar, and
  Leonidas~J. Guibas.
\newblock Caspr: Learning canonical spatiotemporal point cloud representations.
\newblock In \emph{Advances in Neural Information Processing Systems
  (NeurIPS)}, 2020.

\bibitem[Rempe et~al.(2021)Rempe, Birdal, Hertzmann, Yang, Sridhar, and
  Guibas]{rempe2021humor}
Davis Rempe, Tolga Birdal, Aaron Hertzmann, Jimei Yang, Srinath Sridhar, and
  Leonidas~J Guibas.
\newblock Humor: 3d human motion model for robust pose estimation.
\newblock In \emph{Proceedings of the IEEE/CVF international conference on
  computer vision}, pp.\  11488--11499, 2021.

\bibitem[Ryder et~al.(2018)Ryder, Golightly, McGough, and
  Prangle]{ryder2018black}
Tom Ryder, Andrew Golightly, A~Stephen McGough, and Dennis Prangle.
\newblock Black-box variational inference for stochastic differential
  equations.
\newblock In \emph{International Conference on Machine Learning}, pp.\
  4423--4432. PMLR, 2018.

\bibitem[Theodorou(2015)]{theodorou2015nonlinear}
Evangelos~A Theodorou.
\newblock Nonlinear stochastic control and information theoretic dualities:
  Connections, interdependencies and thermodynamic interpretations.
\newblock \emph{Entropy}, 17\penalty0 (5):\penalty0 3352--3375, 2015.

\bibitem[Tong et~al.(2022)Tong, Nguyen-Tang, Tran, and Choi]{tong2022learning}
Anh Tong, Thanh Nguyen-Tang, Toan Tran, and Jaesik Choi.
\newblock Learning fractional white noises in neural stochastic differential
  equations.
\newblock In \emph{Advances in Neural Information Processing Systems}, 2022.

\bibitem[Yang et~al.(2023)Yang, Gao, Lu, Duan, and Liu]{yang2023neural}
Luxuan Yang, Ting Gao, Yubin Lu, Jinqiao Duan, and Tao Liu.
\newblock Neural network stochastic differential equation models with
  applications to financial data forecasting.
\newblock \emph{Applied Mathematical Modelling}, 115:\penalty0 279--299, 2023.

\bibitem[Yang et~al.(2022)Yang, Srivastava, and Mandt]{yang2022diffusion}
Ruihan Yang, Prakhar Srivastava, and Stephan Mandt.
\newblock Diffusion probabilistic modeling for video generation.
\newblock \emph{arXiv preprint arXiv:2203.09481}, 2022.

\bibitem[Zhang et~al.(2023)Zhang, Wei, Zhang, Zhang, and Li]{zhang2023milstein}
Xiao Zhang, Wei Wei, Zhen Zhang, Lei Zhang, and Wei Li.
\newblock Milstein-driven neural stochastic differential equation model with
  uncertainty estimates.
\newblock \emph{Pattern Recognition Letters}, 2023.

\end{thebibliography}
\end{document}